\documentclass[10pt]{article} 
\usepackage[accepted]{tmlr}

\usepackage{amsmath,amsfonts,bm}

\def\eqref#1{equation~\ref{#1}}

\def\1{\bm{1}}

\DeclareMathAlphabet{\mathsfit}{\encodingdefault}{\sfdefault}{m}{sl}
\SetMathAlphabet{\mathsfit}{bold}{\encodingdefault}{\sfdefault}{bx}{n}

\usepackage{alp}
\usepackage{hyperref}
\usepackage{url}
\usepackage{pifont}
\usepackage{xcolor}
\definecolor{revisionblue}{RGB}{0,90,180}
\usepackage{booktabs}
\usepackage{graphicx}
\usepackage{makecell}
\usepackage{algorithm}
\usepackage{algpseudocode}
\usepackage{float}
\usepackage{adjustbox}
\usepackage{multirow}
\usepackage{subcaption}
\usepackage{wrapfig}
\usepackage{amsthm}
\usepackage{pgfplots}
\pgfplotsset{compat=1.18}

\newtheorem{theorem}{Theorem}

\title{Deterministic Prediction as Global Velocity: Flow Matching for Probabilistic Time-Series Forecast}

\author{\name Jiajing Hu \email jiajing.hu@tufts.edu \\
      \addr Department of Computer Science\\
      Tufts University
      \AND
      \name Xu Han \email xu.han@tufts.edu \\
      \addr Department of Computer Science\\
      Tufts University
      \AND
      \name Li-Ping Liu \email liping.liu@tufts.edu\\
      \addr Department of Computer Science\\
      Tufts University}

\def\month{10}  
\def\year{2026} 
\def\openreview{\url{https://openreview.net/forum?id=dDEDwFnnoA}} 

\begin{document}

\maketitle
\begin{abstract}

Accurate uncertainty quantification is critical for many time series forecasting applications, and recent studies have increasingly adopted generative models to capture predictive uncertainty. Existing  approaches often combine deterministic forecasting with generative modeling to improve predictive performance. However, their joint training is typically unstable because the deterministic model is embedded within the generative process, resulting in deep computational graphs and changing the probability transport path. In this work, we propose DV-Flow, a unified flow-matching formulation that incorporates the deterministic predictor as a parallel velocity component rather than as the endpoint of the transport path. This design enables stable end-to-end optimization while subsuming several existing formulations as special cases, allowing direct comparisons among them. Extensive experiments on six real-world datasets demonstrate that DV-Flow consistently improves probabilistic forecasting performance over strong baselines. Our code is available at \url{https://github.com/tufts-ml/dv-flow-time-series}.

\end{abstract}

\section{Introduction}
\label{intro}
Time series forecasting plays a crucial role in domains such as finance, logistics, and healthcare \citep{zhang2024trajectory}. Many downstream decision-making tasks built on forecasts require probabilistic quantification of future trajectories. From a modeling standpoint, this involves learning the conditional distribution of future time series given the past. Accordingly, many deterministic and probabilistic approaches were developed.

Deterministic forecasting models \citep{lin2024cyclenet, wu2021autoformer, zeng2023transformers, fan2022depts, box1970time, sims1980macroeconomics, liu2023itransformer, nie2022time, lim2021temporal} are typically trained to minimize the squared prediction error and therefore approximate the conditional mean of the data distribution. These models excel at producing accurate point forecasts, but their characterization of predictive uncertainty can only come from a Gaussian assumption derived from the mean squared error.

Probabilistic models such as diffusion models \citep{ho2020denoising, ho2022video} and flow-matching frameworks \citep{lipman2022flow, tong2023improving, liu2022flow} have demonstrated strong capabilities in modeling complex, high-dimensional distributions, and have been successfully applied to time series forecasting \citep{rasul2021autoregressive, kerrigan2024functional, tamir2024conditional, zhang2024trajectory, hu2024fm, gao2024auto, kollovieh2024flow, fan2024mg}. Although these models enable sampling of diverse future trajectories and thus support full-trajectory uncertainty quantification, they are not explicitly designed to recover the conditional mean, which remains a central quantity for forecasting.

To leverage the complementary strengths of deterministic and probabilistic models, recent works have explored hybrid formulations that combine both components. Existing approaches typically fall into two categories. Two-stage methods \citep{gong2024cascast, xu2025bridging} first train a deterministic predictor and then fit a probabilistic model conditioned on its output. Although stable, such approaches forgo the potential benefits of joint optimization. Alternatively, several recent methods \citep{li2025diffusion, srivastava2024precipitation, yu2024diffcast, li2024transformer} pursue fully end-to-end training by incorporating the deterministic predictor directly into diffusion or flow-matching formulations. In these approaches, the deterministic model either defines a residual variable or serves as a prior state of the generative path.

A key consequence of the previous joint training designs is that the generative model takes the output of the deterministic predictor as part of its input when constructing the intermediate states of the flow or diffusion process, forming deep computational graphs and yielding a non-stationary regression task during joint optimization. Similar optimization challenges arising from moving targets have been widely recognized in reinforcement learning and adversarial learning \citep{mnih2015human, goodfellow2014generative}. In addition, many existing end-to-end hybrid approaches optimize only the diffusion (or flow-matching) objective, updating both components without explicitly supervising the deterministic predictor, which may not learn the conditional mean and simply increases the network depth. 

In this work, we propose a novel joint-training flow-matching formulation that absorbs the deterministic forecast into the flow-matching velocity as a global velocity component. Specifically, the deterministic forecast gives a global velocity direction, which depends only on the historical context. The original flow-matching velocity is then retained as the local velocity, which is state- and time-dependent and captures the complex distribution of the target time series. Moreover, we add an auxiliary loss that explicitly supervises the deterministic component to leverage the strength of the deterministic point forecaster, while the local velocity captures higher-order variability and uncertainty. This formulation preserves the strengths of both deterministic and probabilistic models and enables principled, stable end-to-end joint training, leading to improved forecast performance and more reliable uncertainty quantification. 

We evaluate our model on six real-world benchmark datasets spanning diverse domains, including electricity consumption and stock prices, where reliable uncertainty quantification is critical for decision-making. Beyond comparisons with SOTA baselines, we identify a unified framework that captures existing hybrid formulations as special cases and conduct a systematic empirical study in a controlled experimental setting. This controlled comparison enables a direct assessment of how different hybrid formulations affect distribution fitting, velocity fitting, and optimization stability.

In summary, our work contributes to research on probabilistic forecasting of time series in three aspects:
\begin{itemize}
\item We introduce a \textbf{D}ual \textbf{V}elocity flow matching framework (DV-Flow) that treats the deterministic predictor as a velocity component and thereby enables stable end-to-end joint training.

\item We analyze the decomposition of probabilistic forecasting into a center prediction and variance estimation through trajectory sampling.

\item We evaluate DV-Flow through controlled experiments and on real-world datasets, demonstrating that the proposed DV-Flow method exhibits stable training dynamics and improves forecast probabilistic distribution performance.
\end{itemize}



\section{Related work}

Deterministic forecasting models generate single-point predictions of future values from past observations without modeling uncertainty. Classical statistical approaches such as ARIMA \citep{box1970time} capture linear dependencies through autoregression, differencing, and moving averages. Recently, Transformer-based models achieve strong results by modeling long-range dependencies with attention, such as Temporal Fusion Transformer (TFT) \citep{lim2021temporal}, PatchTST \citep{nie2022time}, and iTransformer \citep{liu2023itransformer}.

Probabilistic models forecast distributions rather than point estimates, enabling uncertainty quantification for risk-sensitive applications such as finance, supply chain, and healthcare. Early neural approaches like DeepAR \citep{salinas2020deepar} used RNNs to model conditional distributions across related series. More recent work leverages diffusion \citep{ho2020denoising, ho2022video} and flow-matching frameworks for probabilistic forecasting \citep{gao2024auto, fan2024mg}. Examples include TimeGrad \citep{rasul2021autoregressive}, which introduces diffusion-based autoregression, Diffusion-TS \citep{yuan2403diffusion}, which captures seasonality and trends with Fourier-based objectives, TSDiff \citep{kollovieh2024predict} which uses a self-guiding diffusion framework, and TSFlow \citep{kollovieh2024flow}, which applies flow matching with Gaussian Process priors.

To leverage the strength of both deterministic models and probabilistic models, several works have explored hybrid methods that combine them. $\text{D}^3\text{U}$ \citep{li2025diffusion} leverages a pretrained point forecasting model to predict the high-certainty part of the data and trains a separate diffusion model to learn the residual distribution. STVD \citep{srivastava2024precipitation} combines a deterministic downscaling component for a coarser prediction and a diffusion module to learn the residual error to add finer details. DiffCast \citep{yu2024diffcast} trains a global deterministic motion model and a local residual diffusion. CGFM \citep{xu2025bridging} combines a pretrained deterministic model and a flow matching model to learn the residual. TMDM \citep{li2024transformer} leverages a deterministic transformer model to learn the expectation of future trajectory based on context, and uses it as the endpoint prior in the diffusion process. NsDiff \citep{ye2025non} further generalizes this
deterministic-as-prior formulation by pretraining conditional mean and
variance estimators and using their predictions to define a
context-dependent Gaussian endpoint and an uncertainty-aware noise
schedule. These estimators are fixed during diffusion training, resulting
in a cascaded alternative to joint end-to-end optimization. CasCast \citep{gong2024cascast} combines a deterministic component and diffusion in a cascaded manner. It first trains a deterministic model, and then uses its prediction as condition in the downstream diffusion model.

\section{Preliminaries}


\paragraph{Flow matching.} Let $p^*(\by)$ over $\bbR^D$ be the data distribution for a given training set $\mathcal{Y}$. Flow matching models aim to learn an implicit generative model that transforms a simple prior distribution $p(\bepsilon)$, typically a standard Gaussian distribution, to the complex target data distribution. 
The implicit model is essentially an ODE defined by a vector field, which is learned by a neural network. 
%

A typical setting of flow matching is
Rectified Flow \citep{liu2022flow,lipman2022flow}, which defines the vector field with straight lines between data samples and random noise. Let $\by^1 \sim p^*(\by)$ and $\by^0 = \bepsilon \sim p^0 = \mathcal{N}(0,\bI)$. For $\tau \sim \mathcal{U}[0,1]$, $\by^{\tau} = \tau \by^1 + (1-\tau)\by^0$ forms a straight line between $\by^1$ and $\by^0$. At each location $\by^\tau$, the vector field conditioned on two endpoints is $u(\by^\tau, \tau \mid (\by^0, \by^1)) = \by^1-\by^0$. The vector field $v^*(\by^\tau, \tau) = \mathbb{E}_{\by^0, \by^1}[u(\by^\tau, \tau \mid (\by^0, \by^1))]$ by marginalizing over all pairs of $(\by^0, \by^1)$ is the ideal vector field that specifies the distribution $p^*$. The vector field is learned by minimizing the following loss: 
  \begin{align}
     L &= \mathbb{E}_{\tau, \by^0, \by^1} 
\big[ \| v_{\phi}(\mathbf{y}^{\tau}, \tau) -
u(\by^\tau, \tau \mid (\by^0, \by^1)) \|^2 \big] \nonumber \\
&= \mathbb{E}_{\tau, \by^0, \by^1} 
\big[ \| v_{\phi}(\mathbf{y}^{\tau}, \tau) -
v^*(\by^\tau, \tau) \|^2 \big] + 
\mathbb{E}_{\tau, \by^0, \by^1} 
\big[ \| v^*(\mathbf{y}^{\tau}, \tau) -
u(\by^\tau, \tau \mid (\by^0, \by^1)) \|^2 \big].
 \end{align}
 The training objective in the first line is easy to implement. It can be composed as the error in fitting the vector field and a constant term.  The objective computed on a validation set can also be used to compare different models to measure how well they fit the target vector field.

With a trained vector field $v_{\phi}(\by^{\tau},\tau)$, we can draw samples from the learned model by integrating through the ODE path:
\[
\by^1 = \by^0 + \int_{\tau=0}^1 v_{\phi}(\by^{\tau},\tau) d\tau, \quad \by^0 \sim p^0.
\]

\paragraph{Generative models in time series forecasting.} Consider a $D$-dimensional multivariate time series sampled at time points indexed by $t = 1, \ldots, T$. Given past observations $\bx_{1:t_0}$, where $\bx_t \in \mathbb{R}^D$, time series forecasting aims to learn the conditional future data distribution $p^*(\by_{t_0+1:T}|\bx_{1:t_0})$. For simplicity of notation, we will omit the subscript timestamps and use $\by$ to represent $\by_{t_0+1:T}$ and $\bx$ to represent $\bx_{1:t_0}$ in the following sections. Diffusion and flow matching have been successfully applied to time series forecasting. To apply a generative model to the forecasting task, we set $p^*(\by|\bx)$ as the target distribution, which is learned by a diffusion-based or flow-matching generative model specified by a neural network $v(\by^{\tau}, \tau, \bx)$. 

Unlike a typical generative task (e.g., image generation), previous observations $\bx$ provide rich information about the target and require special treatment. Sending $\bx$ to the neural network $v(\by^{\tau}, \tau, \bx)$ is often not enough to extract all the information about $\by$. An important approach is to train a deterministic predictive model $f_\theta(\bx)$ and let $f_\theta(\bx)$ contribute to the parameterization of $v(\by^{\tau}, \tau, \bx)$. 


Residual-based models such as $\text{D}^3\text{U}$ \citep{li2025diffusion}, STVD \citep{srivastava2024precipitation} and DiffCast \citep{yu2024diffcast} first compute the residual $\br_{\theta}^1 = \by - f_\theta(\bx)$ and then learn the distribution of $\br_{\theta}^1$ with a diffusion-based model, which is specified by the network $v_{\phi}(\br_{\theta}^{\tau},\tau, \bx)$. 
In this approach, the neural network cannot see the pattern (e.g., the trend) in the original time series.
  
Deterministic-as-prior models such as 
TMDM \citep{li2024transformer} use a deterministic prediction to specify the noise distribution in the generative model, so the model transforms ``informative'' noise to the target distribution. It runs denoising starting from $f_{\theta}(\bx) + \bepsilon$.


Both types of models can be implemented in the flow-matching framework using similar parameterizations, which we will detail below.

\section{Method}

\begin{figure*}[t]
\begin{center}
\centerline{\includegraphics[width=\textwidth]{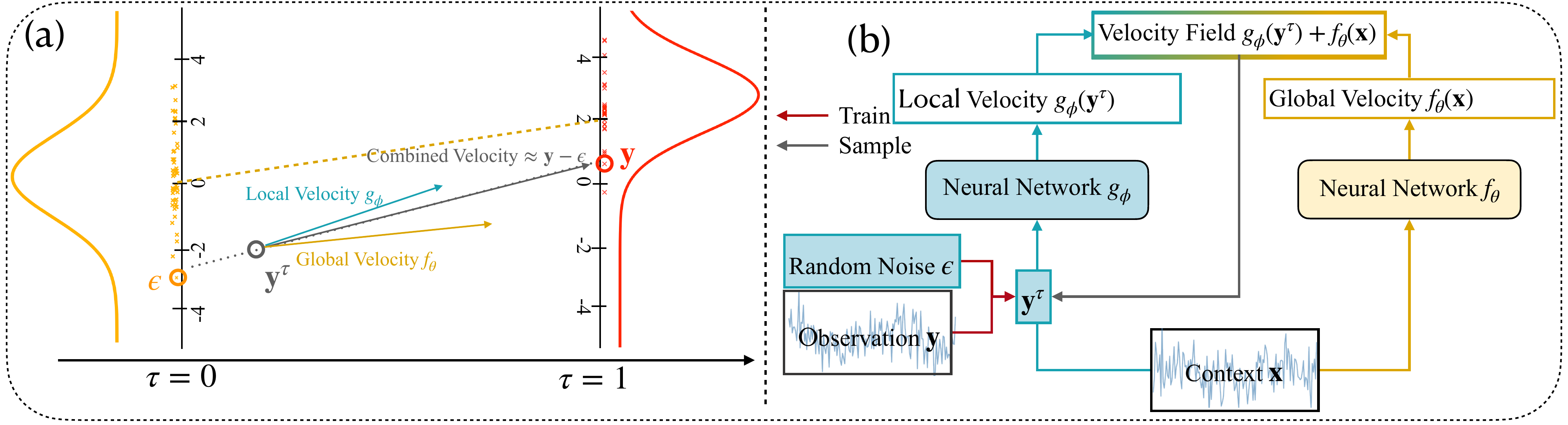}}
\caption{(a) \textbf{Proposed velocity combination framework.} The orange distribution on the left is the random noise at $\tau=0$, and the red distribution on the right is the target distribution at $\tau=1$. Yellow dotted line shows that the deterministic model that predicts the target endpoint can also be interpreted as a global velocity from origin to the center of the target distribution. The velocity from $\bepsilon$ to $\by$ is decomposed into a global velocity and a local velocity. (b) \textbf{Model architecture for DV-Flow.} We removed the nested dependency between $v_{\phi}$ and $f_{\theta}$, and combine them in parallel.}
\label{fig:flowg}
\end{center}
\vskip -0.2in
\end{figure*}

In this section, we introduce DV-Flow (\textbf{D}ual-\textbf{V}elocity Flow), a new hybrid flow-matching framework that combines global and local velocity fields for time series forecasting.
In this method, we use the deterministic component as a \emph{velocity term} rather than as an endpoint of the transport path. 

We still use Gaussian noise $\bepsilon$ as the start of the ODE in the flow matching model, and $p(\by | \bx)$ as the target distribution.
In particular, we parameterize the velocity as: 
\[
v_{\theta, \phi}(\by^\tau, \tau \mid \bx) = f_\theta(\bx) + g_\phi(\by^\tau, \tau, \bx),\label{eq:model-def}
\] 
where $f_\theta(\bx)$ aims for a deterministic prediction, which can be interpreted as a constant velocity, and $g_{\phi}$ is the neural network for the local velocity. A visual demonstration can be found in Figure~\ref{fig:flowg}. With $\by^{\tau} = (1-\tau)\bepsilon + \tau \by$, 
the loss function is:
\begin{equation}
    \mathcal{L} = \mathbb{E}_{\bepsilon,\by,\bx,\tau}\left[
\left\|
(\by - \bepsilon)
- v_{\theta, \phi}(\by^\tau, \tau \mid \bx)
\right\|^2
\right].
\end{equation}
From the standard flow matching \citep{lipman2022flow}, the optimal vector field $v_*(\by^\tau, \tau | \bx)$ for an input $\bx$ is
\begin{align}
v_*(\by^\tau, \tau | \bx) &= \mathbb{E}_{\bepsilon,\by}\left[\by - \bepsilon \mid \by^{\tau},\bx\right],
\end{align}
and the loss can be decomposed into $\mathcal{L} = D + C$ with
\begin{align}
D = \mathbb{E}_{\bepsilon,\by, \bx,\tau}\left[ \|v_{\theta, \phi}(\by^\tau, \tau \mid \bx) - v_*(\by^\tau, \tau | \bx) \|^2\right], \quad C= \mathbb{E}_{\bepsilon,\by,\bx,\tau}\left[\|(\by - \bepsilon) - v_*(\by^\tau, \tau | \bx)\|^2\right]. \label{eq:Cterm}
\end{align}
In this decomposition, the second term $C$ is a constant: it represents the difference between the marginal vector field and the conditional vector fields. Therefore, minimizing the loss is equivalent to minimizing the difference between the learned vector field $v_{\theta, \phi}$ and the true marginal vector field $v_*$.   



\paragraph{A comparison to other hybrid formulations.} \label{sec: comparable-loss} We present two alternative parameterizations that incorporate a deterministic forecasting component into the flow matching model. These approaches are adapted from existing diffusion-based forecasting methods and can be viewed as intermediate formulations bridging previous methods and the proposed approach.

In the first case, the model learns the distribution of the residual. Let $\br^0 = \bepsilon$ and $\br^1 = \by - f_\theta(\bx)$.
The expected vector field has the following relationship to the proposed method. 
\[
v_*^r(\br^\tau, \tau | \bx) = \mathbb{E}_{\bepsilon, \by}[\by - f_\theta(\bx) - \bepsilon \mid \br^{\tau}, \bx] =  
v_*(\by^\tau, \tau | \bx) - f_\theta(\bx).  
\]
Here $\br^\tau =  \by^\tau - \tau f_\theta(\bx)$. 
Suppose a neural network $v^r(\br^\tau, \tau, \bx)$ is parameterized to fit
$v_*^r$, then the training loss can still be decomposed into $D^r + C$, with 
$D^r = \mathbb{E}_{\bepsilon, \by}\left[\| v^r(\br^\tau, \tau, \bx) -  v_*^r(\br^\tau, \tau | \bx)\|^2 \right]$, and the term $C$ being exactly the same as in \eqref{eq:Cterm}. 

It has an equivalent form $v^{r,y}(\by^\tau, \tau, \bx) = v^{r}(\br^\tau, \tau, \bx) + f_\theta(\bx)$ with $v_*$ as the fitting target. Then we also have $D^r = \mathbb{E}_{\bepsilon, \by}\left[\| v^{r,y}(\by^\tau, \tau, \bx) -  v_*(\by^\tau, \tau | \bx)\|^2 \right]$. When $v^{r,y}(\by^\tau, \tau, \bx) = v_{\theta,\phi}(\by^\tau, \tau, \bx)$, the model defined on the residuals has the same distribution as the previous model, and the learning objective is also the same.  

In the second case, the noise is informative and contains the deterministic prediction $\bh^0 = (f_\theta(\bx) + \bepsilon)$ and $\bh^1 = \by$.
In this case, the optimal vector field is 
\[
v^h_*(\bh^\tau, \tau | \bx) = \mathbb{E}_{\bepsilon, \by}[\by - ( f_\theta(\bx) + \bepsilon) \mid \bh^{\tau}, \bx] =  
v_*(\by_\tau, \tau | \bx) - f_\theta(\bx).
\]
Here $\bh^\tau = \by^\tau + (1 - \tau) f_\theta(\bx)$. The neural network $v^h(\bh^\tau, \tau | \bx)$ in this setting has the same fitting target as $v^r(\br^\tau, \tau | \bx)$, although the input is different.  
The learning objective can be decomposed by $D^h + C$, with the constant term $C$ the same as the other two cases. The equivalent form $v^{h,y}(\by^\tau, \tau, \bx) = v^h(\bh^\tau, \tau, \bx) + f_\theta(\bx)$ should fit $v_*$, and $D^h$ also measures the error between $v^{h,y}$ and $v_*$.

The three parameterizations are equivalent on the output side. However, they take different inputs and behave differently in model training. For parameterization $v^r(\br^\tau, \tau | \bx)$, the residual $\br^1 = \by - f_\theta(\bx)$ depends on the output of the deterministic model. The same is true for $v^h(\bh^\tau, \tau | \bx)$ where $\bh^0 = \bepsilon + f_\theta(\bx)$. In both cases, the neural network that computes the vector field ($v^r$ or $v^h$) is deeply coupled with $f_\theta$. If they are jointly trained, the updates of $f_\theta$ continuously change the fitting targets of $v^r$ or $v^h$, creating additional difficulties in the training procedure. Furthermore, $f_\theta$ is not trained to make point-wise predictions, so it merely increases the depth of the network ($v^r$ or $v^h$) for vector computation.  

Since we can derive $v_{\theta,\phi}$, $v^{r,y}$, and $v^{h,y}$ from the three parameterizations, all of which fit $v_*$, we can measure their fitting errors by comparing their learning objectives on a validation set. In our experiment, we examine how they behave differently in model training.

\paragraph{Joint training.} In our new method, we jointly train the global and local velocity components in a combined objective:   
\begin{equation}
\label{eq:coflow}
    \mathcal{L} = \mathbb{E}_{\bepsilon,\by,\bx,\tau}\left[
\left\|
(\by - \bepsilon)
- \big(f_{\theta}(\bx) + g_\phi(\by^{\tau}, \tau, \bx) \big)
\right\|^2
\right] + \lambda \mathbb{E}_{\bx,\by}
\big[
\mathcal{L}_f(\by, f_\theta(\bx))
\big].
\end{equation}
The point-wise loss term $\mathcal{L}_f(\by, f_\theta(\bx))$ explicitly supervises the deterministic predictor. Common choices include MSE and MAE. When the conditional mean $\mathbb{E}[\by | \bx]$ is in a location with a high probability density, training $f_\theta(\bx)$ towards the distribution mean helps to improve  overall performance. However, when the mean is at a low-density location (e.g., in a bimodal distribution), it is not helpful for $f_\theta(\bx)$ to capture the mean.     
With our proposed parameterization method, the two neural networks $f_{\theta}(\bx)$ and $g_\phi(\by^{\tau}, \tau, \bx)$ do not directly depend
on each other, so the training is more stable than the other two methods.  

\paragraph{Relationship with distribution fitting.} We now study the proposed model from the perspective of mean-variance decomposition and inspect its relationship with predictive uncertainty. 

\begin{theorem} 
Let $v_{\theta,\phi}$ be the flow matching model defined in \eqref{eq:model-def}, and assume $v_{\theta,\phi}$ learns the true vector field $v_*$ for a given $\bx$. Let the time dependent vector field $g'(\br, \tau) = g_\phi(\br + \tau f_\theta(\bx), \tau, \bx)$, then the ODE defined by $g'$ pushes the Gaussian noise $\epsilon$ to the distribution with samples $\by - f_\theta(\bx)$. In particular, if $f_\theta(\bx) = \mathbb{E}_{\by}[\by \mid \bx]$, then the ODE defined by $g'$ pushes $\bepsilon$ to a zero-mean distribution. Alternatively, if $f_\theta(\bx)$ is a
coordinate-wise conditional median of $\by$ given $\bx$, then the ODE
defined by $g'$ pushes $\bepsilon$ to a residual distribution for which a coordinate-wise conditional median is zero.
\end{theorem} 
Theorem 1 is better understood from the residual formulation. We provide the proof in the Appendix~\ref{sec:appendixA}. 
We can check the theorem with two simple cases. Suppose $g_{\phi}(\by^\tau, \tau, \bx) = \bzero$, when $f_{\theta}(\bx)$ is supervised with MSE loss, then: 
\begin{align}
\mathbb{E}[\by^1|\bx] =  
 f_{\theta}(\bx)
,\quad \mathrm{Var}[\by^1|\bx] = \bI. \label{eq:example1}
\end{align}
In this case, $f_{\theta}(\bx)$ carries the zero-mean noise to a new location without changing its variance. In a more complex case 
where $g_{\phi}(\by^\tau, \tau, \bx) = \bA(\by^{\tau} - \tau f_\theta(\bx))$, then $g_{\phi}(\br, \tau) = \bA \br$, then:
\begin{align}
    \mathbb{E}[\by^1|\bx] &=f_{\theta}(\bx), \quad \mathrm{Var}[\by^1|\bx] = e^\bA(e^\bA)^\top.  \label{eq:example2}
\end{align}
However, if the linear function is applied to $\by^\tau$: $g_{\phi}(\by^\tau, \tau, \bx) = \bA\by^{\tau}$, then $g_{\phi}$ affects both the mean transformation and the covariance propagation, while 
$f_\theta(\bx)$ contributes only to the mean through the linear response $\bA^{-1}(e^\bA-I)f_{\theta}(\bx)$:
\begin{align}
    \mathbb{E}[\by^1|\bx]  = \bA^{-1}(e^\bA-I)f_{\theta}(\bx),  \quad
    \mathrm{Var}[\by^1|\bx] = e^\bA(e^\bA)^\top.
\end{align}
When $f_{\theta}(\bx)=\operatorname{Median}[\by^1|\bx]$, the two configurations of $g_\phi$ in \eqref{eq:example1} and \eqref{eq:example2} yield similar results, except that  $f_\theta(\bx)$ becomes the median. All derivations are given in the Appendix~\ref{sec:appendixA}. With a complex neural network $g_\phi$, the model can capture non-linear variances in the data and should perform better than a Gaussian model. 

\paragraph{Model Architecture.} 
\label{sec:SATN}
\begin{wrapfigure}{r}{0.42\textwidth}
    \centering
    \vspace{-10pt}
    \includegraphics[width=0.40\textwidth]{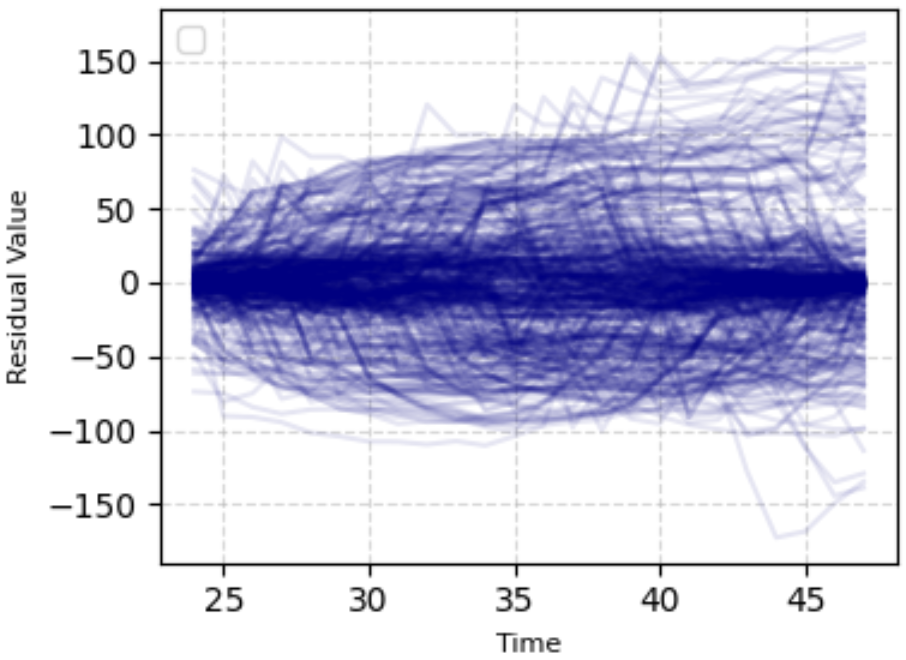}
    \caption{Residuals on the Electricity dataset after removing seasonality and centering by the last context observation. The variance increases with the forecasting horizon, motivating our time-adaptive normalization scheme.}
    \label{fig:electricity_residual}
    \vspace{-10pt}
\end{wrapfigure}
In the task of time series forecasting, we adopt iTransformer \citep{liu2023itransformer} as the backbone network for both $v_{\phi}(\by^{\tau}, \tau , \bx)$ and $f_{\theta}(\bx)$.

To better capture the correlation between neighboring steps, we sample noise $\bepsilon$ from a zero-mean Gaussian–process specified with Ornstein–Uhlenbeck (exponential) kernel $K (t,t') = \exp{(-\frac{|d|}{l})}$, where $d=t-t'$ is the temporal lag and $l$ is a non-negative correlation-length hyperparameter.

Proper normalization is essential for stable training. We propose Seasonal-Adjusted Time-Adaptive Normalization (SATN), which first removes seasonal components from the data and then applies a time-dependent normalization scheme ($\by_t \leftarrow \frac{\by_t - \bx_0}{\bsigma_t}$). Motivated by Brownian motion, we treat the last observation in the context window as the starting point of the forecasting trajectory and center the prediction window accordingly. Under this view, we assume that the residual variance grows with the forecasting horizon and normalize each time step using a linearly increasing variance schedule estimated from the data. Figure~\ref{fig:electricity_residual} illustrates the residuals of the Electricity dataset after removing seasonality and centering the prediction window by the last observation in the context sequence. We observe that the residual variance increases monotonically with the forecasting horizon: predictions closer to the context window exhibit smaller deviations, while those farther away become progressively more dispersed. This behavior motivates our time-adaptive normalization strategy, which explicitly accounts for the growth of uncertainty over time. This transforms the prediction sequence into a process with approximately unit marginal variance, improving alignment with the Gaussian prior used in flow matching and mitigating variance disparities across time steps. Additional details are provided in Appendix~\ref{sec:appendix_SATN}.

In the sampling procedure, the integration is computed from the Dormand–Prince 45 \citep{dormand1980runge} ODE solver. By solving Equation~\ref{eq:sample}, we obtain a transformed sample $\by^1$. This sample is then reverse scaled with SATN. A complete sampling procedure is summarized in Algorithm~\ref{alg:sampling}. Note that even though $f_{\theta}(\bx)$ is used at every step, since it is $\tau$-invariant, we only need to compute it once up front, and it can be treated as a fixed value in the integral. Therefore, adding the deterministic network does not slow down inference.
\begin{equation}
\label{eq:sample}
\hat{\by} = f_{\theta}(\bx); \quad 
\by^1 = \bepsilon + \int_{\tau=0}^1 v_{\theta, \phi}(\by^\tau, \tau, \bx) d\tau = \bepsilon + \int_{\tau=0}^1 (\hat{\by} +  g_\phi(\by^\tau, \tau, \bx)) d\tau 
\end{equation}

\section{Experiments}


In this section, we evaluate the ability of the proposed model to forecast time series and quantify uncertainty using datasets from real-world applications. In Section~\ref{sec:baselines} we compare DV-Flow with existing baselines, in Section~\ref{sec:hybrid} we study different hybrid formulations under our unified framework, and in Section~\ref{sec:pointwise_loss} we investigate the effect of the point-wise loss term in the objective. 

\paragraph{Datasets.} We use six widely adopted multivariate benchmark datasets: five from GluonTS \citep{alexandrov2019gluonts}-Electricity \citep{dua2017uci}, Exchange \citep{lai2018modeling}, Traffic \citep{dua2017uci}, Solar \citep{lai2018modeling}, KDDCup \citep{stolfo1999kdd}-and one finance dataset consisting of minute-level log returns of 486 stocks from S\&P 500 index. Dataset statistics can be found in Table~\ref{tab:datasets}. We use the standard train and test splits provided by GluonTS.

\begin{table}[H] 
  \caption{Dataset Information.}
  \label{tab:datasets}
  \centering
  \small
  \setlength{\tabcolsep}{6pt}   
  \begin{tabular}{lcccccc}         
    \toprule
    \bfseries Dataset & \bfseries Timesteps & \bfseries Channels & \bfseries Frequency & \bfseries Context & \bfseries Forecast & \bfseries Seasonal-window\\
    \midrule
    Electricity & 5{,}833   & 370 & Hour & 24 &24 & 24 \\
    Exchange    & 6{,}071   &   8 & Day  & 30 &30 & 30\\
    Traffic     & 4{,}001   & 963 & Hour & 24 &24&24\\
    Solar       & 7{,}009   & 137 & Hour & 24 &24&24\\
    KDDCup      & 8{,}448   & 270 & Hour & 24 &24&24\\
    SP500       & 255{,}840 & 486 & Minute & 30 &30&-\\
    \bottomrule
  \end{tabular}
\end{table}

\paragraph{Baselines.} We compare our model against one traditional statistical model ARIMA \citep{box1970time}, three deterministic models TFT \citep{lim2021temporal}, PatchTST \citep{nie2022time}, and iTransformer \citep{liu2023itransformer}; five probabilistic models: TimeGrad \citep{rasul2021autoregressive}, Diffusion-TS \citep{yuan2403diffusion}, DeepAR \citep{salinas2020deepar},  TSDiff \citep{kollovieh2024predict}, and TSFlow \citep{kollovieh2024flow}; and two hybrid models: TMDM \citep{li2024transformer} and $\text{D}^3\text{U}$ \citep{li2025diffusion}. To evaluate predictive probabilities, we sample $100$ trajectories for each probabilistic model on the test set. We use $150$ inference steps. 

\paragraph{Implementation details.} For both networks, the hidden dimensions are set to 1024, with 8 attention heads and 4 stacked attention layers. Models are trained using the AdamW optimizer \citep{loshchilov2017decoupled}, with a learning rate of $2e^{-4}$ and a weight decay of $1e^{-5}$. For the local velocity network, the flow-matching time step $\tau$ is encoded using Fourier positional embeddings \citep{vaswani2017attention} into $\mathbb{R}^{32}$, which is then appended to the input as an additional token in the transformer sequence. We use 8 workers on a single NVIDIA A100 GPU node for all experiments. Mean and standard deviation for all models are obtained from three independent runs.  We use $\lambda=0.3$ across all datasets for point-wise loss. We use the Dormand–Prince 45 ODE solver with an absolute tolerance of $\mathrm{atol}=10^{-6}$ and a relative tolerance of $\mathrm{rtol}=10^{-3}$.

\paragraph{Evaluation metrics.}  We use CRPS \citep{matheson1976scoring} to jointly assess the calibration
and sharpness of the predictive distribution represented by the forecast
samples. Following \cite{rasul2021autoregressive}, we use $\text{CRPS}_{\text{sum}}= \mathbb{E}_t \big[\text{CRPS}(\hat{F}^{t}_{\text{sum}}(\cdot),  \sum_{i=1}^D y_{i,t}) \big]$, which evaluates the predictive calibration and sharpness of the sum of each time series at each time step and then aggregates the CRPS values across time by averaging them. Here, $\hat{F}^{t}_{\mathrm{sum}}(\cdot)$ denotes the empirical cumulative distribution function (CDF) of the sum of the variables at time $t$. The empirical CDF is computed from samples from the predictive model. 


In this work, we propose a new evaluation metric Corr-$\text{CRPS}_{\text{mean}}$ that evaluates the uncertainty quantification of the correlation between time series, as the correlation is an important consideration in applications such as stock investment. Note that $\by \in\mathbb{R}^{(T-t_0)\times D}$ denote the observed $D$-dimensional time series in the future window. 
Let $Corr_{ij}$ denote the correlation between the $i$-th and $j$-th univariable time series in this time window. 
Here $Corr_{ij}$ can be viewed as a sample from the true distribution of correlation values. 
We then compute samples $\{\widehat{Corr}^{s}_{ij}: s = 1, \ldots, S\}$ from forecasting samples and estimate the CDF $\hat{F}_{ij}(\cdot)$ from these samples. 
The CRPS metric computed from $Corr_{ij}$ and $\hat{F}_{ij}(\cdot)$ measures how well forecast samples capture uncertainty in the correlation between time series $i$ and $j$. To obtain a single evaluation metric, we aggregate the CRPS over all $ij$ pairs and all $y$ instances in the test set. Formally, Corr-$\text{CRPS}_{\text{mean}}$ is defined as the follow where where $n$ indexes test instances:
$$
\operatorname{Corr\text{-}CRPS}_{\mathrm{mean}}
= 
\frac{2}{ND(D-1)} \sum_{n=1}^N
\sum_{i,j}
\operatorname{CRPS}
\left(
\widehat{F}_{ij}^{(n)},
Corr_{ij}^{(n)}
\right).
$$

We further evaluate the quantile forecasts derived from the sampled
predictive distributions. Following \citet{lim2021temporal}, we use
$q$-Risk, a normalized quantile-loss metric that evaluates the accuracy
of the predicted conditional $q$-quantile by penalizing
underestimation and overestimation. 
We report q-Risk at $q=0.9$, denoted as $Q90$. Details of the metrics can be found in Appendix~\ref{sec:evaluation}.

\subsection{Compare with existing baselines}
\label{sec:baselines}
In this section, we evaluate DV-Flow against the baseline methods with three metrics: CRPS on time series forecasts, CRPS on correlation among forecast time series dimensions, and quantile loss. In all tables, the best results are highlighted in $\textbf{bold}$, while the second best are \underline{underlined}.

\begin{table}[t]
    \caption{Test $\text{CRPS}_{\text{sum}} \downarrow$}
\label{crps_results}
\begin{adjustbox}{width=0.95\textwidth}
\begin{tabular}{lcccccc}
\toprule
Method & Electricity & Exchange & Traffic & Solar & KDDCup & SP500 \\
\midrule
ARIMA        & $0.6058 \pm 0.000$ & $0.0062 \pm 0.000$ & $0.7265 \pm 0.000$ & $0.5450 \pm 0.000$ & $\underline{0.1111} \pm 0.000$ & $1.0768 \pm 0.000$ \\
TFT          & $0.0276 \pm 0.001$ & $0.0085 \pm 0.002$  & $0.0582 \pm 0.026$ & $0.9509 \pm 0.357$ & $0.9375 \pm 0.110$ & $\underline{0.8260} \pm 0.072$ \\
PatchTST     & $0.0231 \pm 0.005$ & $0.0109 \pm 0.005$ & $0.0906 \pm 0.022$ & $0.8958 \pm 0.024$ & $0.9228 \pm 0.065$ & $0.9430 \pm 0.012$ \\
iTransformer & $0.0228 \pm 0.002$ & $0.0087 \pm 0.003$ & $0.0689 \pm 0.017$ & $0.9078 \pm 0.034$ & $0.5902 \pm 0.181$ & $0.8968 \pm 0.106$ \\
DeepAR       & $0.0373 \pm 0.012$ & $0.0077 \pm 0.002$ & $\underline{0.0434} \pm 0.004$ & $0.4020 \pm 0.003$ & $0.1600 \pm 0.066$ & $0.9455 \pm 0.020$ \\
Diffusion-TS & $0.0773 \pm 0.004$ & $0.0090 \pm 0.003$ & $0.5771 \pm 0.014$ & $0.5447 \pm 0.024$ & $0.4395 \pm 0.189$  & $0.9063 \pm 0.015$  \\
TimeGrad     & $\underline{0.0206} \pm 0.001$ & $0.0060 \pm 0.001$ & $0.0440 \pm 0.006$ & $0.3565 \pm 0.016$ & $0.1165 \pm 0.033$ & $0.9089 \pm 0.003$ \\
TSDiff & $0.0490 \pm 0.000$ & $0.0110 \pm 0.001$ & $0.0980 \pm 0.002$ & $0.3580 \pm 0.020$ & $0.3778 \pm 0.037$ & $0.9199 \pm 0.012$ \\

 TSFlow & $0.0239 \pm 0.000$ &  $0.0049 \pm 0.000$ & $0.1307 \pm 0.001$ &  $\underline{0.3408} \pm 0.036$ & $0.1179 \pm 0.028$ & $0.9786 \pm 0.047$  \\
TMDM & $0.0440 \pm 0.002$ & $\underline{0.0047} \pm 0.000$ & $0.0650 \pm 0.003$ & $1.2461 \pm 0.391$ & $0.1463 \pm 0.013$ & $0.8705 \pm 0.033$ \\
D$^3$U & $0.0261 \pm 0.002$ & $0.0059 \pm 0.001$ & $0.0586 \pm 0.004$ & $1.5622 \pm 0.296$ & $0.1114 \pm 0.020$ & $0.9750 \pm 0.034$ \\
\midrule
DV-Flow      & $\mathbf{0.0197} \pm 0.000$ & $\mathbf{0.0046} \pm 0.000$ & $\mathbf{0.0386} \pm 0.005$ & $\mathbf{0.3280} \pm 0.032$ & $\mathbf{0.0578} \pm 0.001$ & $\mathbf{0.7958} \pm 0.018$\\
\bottomrule
\end{tabular}
\end{adjustbox}
\end{table}

\begin{table}[t]
    \caption{Testset $\text{Corr-CRPS}_{\text{mean}} \downarrow$}
\label{corr_crps_results}
\begin{adjustbox}{width=0.95\textwidth}
\begin{tabular}{lcccccc}
\toprule
Method & Electricity & Exchange & Traffic & Solar & KDDCup & SP500 \\
\midrule
TFT          & $0.3769 \pm 0.029$ & $1.1208 \pm 0.142$  & $0.1356 \pm 0.001$ & $0.1178 \pm 0.000$ & $1.5412 \pm 0.052$ & $1.2477 \pm 0.299$ \\
PatchTST     & $0.3115 \pm 0.013$ & $0.8118 \pm 0.012$ & $0.1834 \pm 0.012$ & $0.8128 \pm 0.002$ & $0.8203 \pm 0.004$ & $0.7485 \pm 0.001$ \\
iTransformer & $0.2144 \pm 0.003$ & $0.0324 \pm 0.010$ & $0.1741 \pm 0.052$ & $0.8078 \pm 0.006$ & $0.8173 \pm 0.009$ & $0.7566 \pm 0.000$ \\
DeepAR       & $0.2242 \pm 0.107$ & $0.7804 \pm 0.040$ & $0.2724 \pm 0.343$ & $0.2074 \pm 0.018$ & $0.7160 \pm 0.004$ & $0.7554 \pm 0.000$ \\
Diffusion-TS & $0.2378 \pm 0.006$ & $0.8183 \pm 0.021$ & $0.2339 \pm 0.002$ & $0.1049 \pm 0.007$ & $1.3030 \pm 0.107$  & $1.2116 \pm 0.050$  \\
TimeGrad     & $0.1965 \pm 0.006$ & $0.6929 \pm 0.010$ & $0.1274 \pm 0.009$ & $0.0890 \pm 0.002$ & $0.6900 \pm 0.046$ & $0.7484 \pm 0.001$ \\
TSDiff & $\mathbf{0.1616} \pm 0.001$ & $0.7831 \pm 0.175$ & $\mathbf{0.0767} \pm 0.002$ & $\underline{0.0759} \pm 0.005$ & $0.6818 \pm 0.022$ & $0.7475 \pm 0.004$ \\
TSFlow & $0.1824 \pm 0.000$ & $\underline{0.6717} \pm 0.016$ & $\underline{0.0917} \pm 0.000$ & $0.0796 \pm 0.010$ & $\underline{0.6589} \pm 0.018$ & $0.7447 \pm 0.002$ \\
TMDM & $1.2088 \pm 0.401$ & $0.8092 \pm 0.042$ & $0.4789 \pm 0.311$ & $0.6877 \pm 0.046$ & $0.7815 \pm 0.031$ & $0.7528 \pm 0.001$ \\
D$^3$U & $0.2798 \pm 0.006$ & $0.8367 \pm 0.017$ & $0.1228 \pm 0.001$ & $0.1056 \pm 0.013$ & $0.7036 \pm 0.016$ & $\underline{0.7376} \pm 0.002$ \\
\midrule
DV-Flow      & $\underline{0.1761} \pm 0.002$ & $\mathbf{0.6380} \pm 0.013$ & $0.0934 \pm 0.003$ & $\mathbf{0.0752} \pm 0.004$ & $\mathbf{0.6544} \pm 0.007$ & $\mathbf{0.6256} \pm 0.002$\\
\bottomrule
\end{tabular}
\end{adjustbox}
\end{table}

\paragraph{$\text{CRPS}_{\text{sum}}$.} First, we compare DV-Flow against baselines in terms of marginal distribution probability. Table~\ref{crps_results} presents our main results on six datasets evaluated by $\text{CRPS}_{\text{sum}}$. DV-Flow substantially outperforms the other common baselines on all six datasets. This robust performance is evident across diverse domains, from the highly periodic patterns in Electricity to the more stochastic behavior in Exchange-Rate, demonstrating the wide applicability of our method. Specifically, the significant advantage in $\text{CRPS}_{\text{sum}}$ indicates that DV-Flow fits the distribution the best, highlighting the effectiveness of our hybrid velocity formulation with stable joint training.

\paragraph{$\text{Corr-CRPS}_{\text{mean}}$.} Our method achieves the best $\text{Corr-CRPS}_{\text{mean}}$ performance on 4 out of 6 datasets, demonstrating its strong ability to capture cross-dimensional dependency structures in multivariate forecasting. Unlike vanilla $\text{CRPS}_{\text{sum}}$ , which evaluates each dimension independently, $\text{Corr-CRPS}_{\text{mean}}$  explicitly measures whether the predicted samples preserve the correlation patterns observed in the ground truth. The consistent improvements across diverse datasets suggest that our model not only produces better marginal forecasts, but also better models the joint dynamics among variables. This is particularly important in applications such as finance and energy systems, where correlations between dimensions play a critical role in downstream decision-making.

\paragraph{Q90.} We then compare DV-Flow against baselines in terms of quantile risk. The results are presented in Table~\ref{tab:uq_metrics_vertical_column}, where DV-Flow demonstrates clear superiority over all baseline models. It ranks first on four of the six datasets and second best for the rest two. This result underscores the model's capacity to capture $90\%$ of future outcomes with tight and informative prediction intervals, avoiding overly conservative or overly wide uncertainty estimates. For this analysis, we re-evaluated all baselines using their official implementations to ensure a fair comparison; iTransformer's configuration was specifically aligned with our backbone for consistency. Details in Appendix~\ref{uq-exp}

\begin{table}[t]
    \caption{Testset Q90 $\downarrow$}
\label{tab:uq_metrics_vertical_column}
\begin{adjustbox}{width=0.95\textwidth}
\begin{tabular}{lcccccc}
\toprule
Method & Electricity & Exchange & Traffic & Solar & KDDCup & SP500 \\
\midrule
TFT     & $0.0566 \pm 0.015$ & $0.0117 \pm 0.006$ & $0.1516 \pm 0.023$ & $0.2116 \pm 0.042$ & $0.7132 \pm 0.012$ & $0.5868 \pm 0.028$ \\
PatchTST     & $0.0407 \pm 0.002$ & $0.0068 \pm 0.002$ & $0.1150 \pm 0.010$ & $0.4596 \pm 0.079$ & $0.6739 \pm 0.056$ & $0.5694 \pm 0.029$ \\
iTransformer     & $0.0350 \pm 0.001$ & $0.0078 \pm 0.002$ & $0.1273 \pm 0.018$ & $0.3029 \pm 0.036$ & $0.4976 \pm 0.044$ & $0.5492 \pm 0.004$ \\
DeepAR       & $0.0380 \pm 0.003$ & $\underline{0.0050} \pm 0.000$ & $\mathbf{0.0911} \pm 0.001$ & $0.3194 \pm 0.017$ & $0.3992 \pm 0.074$ & $\underline{0.5384} \pm 0.008$ \\
Diffusion-TS & $0.1112 \pm 0.008$ & $0.0258 \pm 0.004$ & $0.8279 \pm 0.029$ & $0.7985 \pm 0.060$ & $0.4212 \pm 0.056$  & $0.7122 \pm 0.029$  \\
TimeGrad     & $0.0412 \pm 0.004$ & $0.0082 \pm 0.000$ & $0.1120 \pm 0.021$ & $0.2181 \pm 0.010$ & $0.3048 \pm 0.036$ & $0.5565 \pm 0.005$ \\
TSDiff & $0.0403 \pm 0.003$ & $0.0166 \pm 0.014$ & $0.1411 \pm 0.004$  & $0.2916 \pm 0.018$ & $0.2804 \pm 0.043$ & $0.5472 \pm 0.013$ \\
TSFlow & $\mathbf{0.0329} \pm 0.000$ &  $0.0053 \pm 0.000$ & $0.1454 \pm 0.001$  & $\underline{0.1988} \pm 0.028$ & $0.2859 \pm 0.008$ &  $0.5636 \pm 0.037$ \\
TMDM & $0.1212 \pm 0.001$ & $0.0101 \pm 0.022$ & $0.3885 \pm 0.016$ & $1.9536 \pm 0.066$ & $0.4608 \pm 0.025$ & $0.6723 \pm 0.003$ \\
D$^3$U & $0.0423 \pm 0.002$ & $0.0058 \pm 0.001$ & $0.1360 \pm 0.001$ & $0.7703 \pm 0.090$ & $\underline{0.2529} \pm 0.044$ & $0.5898 \pm 0.004$ \\
\midrule
DV-Flow      & $\underline{0.0348} \pm 0.002$ & $\mathbf{0.0041} \pm 0.000$ & $\underline{0.1054} \pm 0.003$ & $\mathbf{0.1866} \pm 0.025$ & $\mathbf{0.2054} \pm 0.007$ & $\mathbf{0.5235} \pm 0.008$\\
\bottomrule
\end{tabular}
\end{adjustbox}
\end{table}

\subsection{Dissecting different hybrid formulations}
\label{sec:hybrid}

In this section, we conduct a thorough investigation of different settings in flow matching for time series forecasting. We mainly consider the setup of the flow-matching model for learning the target distribution. Across all three hybrid formulations, we use the same SATN preprocessing, Gaussian-process prior, two iTransformer networks, training budget, optimization settings, and point-wise auxiliary loss.

\paragraph{Distribution fitting.}
 
\begin{table}[t]
\centering
\caption{Test set performance across hybrid formulations. $\downarrow$}
\label{tab:hybrid_results}
\scriptsize

\begin{adjustbox}{width=0.95\textwidth}
\begin{tabular}{llrrrrrrrrrr}
\toprule
Dataset & Metric 
& Residual$_{0.3}$ 
& Residual$_0$ 
& DP$_{0.3}$ 
& DP$_0$  
& DV-Flow\\
\midrule
\multirow{2}{*}{Electricity}
& $\text{CRPS}_{\text{sum}}$ & $0.0326 \pm 0.016$ & $0.0288 \pm 0.017$ & $0.0239 \pm 0.001$ & $0.0264 \pm 0.001$ &  $\mathbf{0.0197} \pm 0.000$ \\
& Q90  & $0.0410 \pm 0.009$ & $0.0351 \pm 0.001$ & $0.0367 \pm 0.002$ & $0.0364 \pm 0.001$ & $\mathbf{0.0348} \pm 0.002$\\
\midrule
\multirow{2}{*}{Exchange}
& $\text{CRPS}_{\text{sum}}$ & $0.0052 \pm 0.000$ & $0.0052 \pm 0.000$ & $0.0049 \pm 0.000$ & $0.0049 \pm 0.001$ & $\mathbf{0.0046} \pm 0.000$ \\
& Q90  & $0.0042 \pm 0.000$ & $0.0043 \pm 0.000$ & $0.0041 \pm 0.000$ & $0.0041 \pm 0.001$ & $\mathbf{0.0041} \pm 0.000$ \\
\midrule
\multirow{2}{*}{Traffic}
& $\text{CRPS}_{\text{sum}}$ & $0.0604 \pm 0.004$ & $0.0665 \pm 0.028$ & $0.0561 \pm 0.020$ & $0.1204 \pm 0.021$ & $\textbf{0.0386} \pm 0.005$ \\
& Q90  & $0.1322 \pm 0.032$ & $0.1151 \pm 0.012$ & $0.1187 \pm 0.016$ & $0.1762 \pm 0.031$ & $\mathbf{0.1054} \pm 0.003$\\
\midrule
\multirow{2}{*}{Solar}
& $\text{CRPS}_{\text{sum}}$ & $0.4298 \pm 0.069$ & $0.4262 \pm 0.033$ & $0.3467 \pm 0.038$ & $0.3817 \pm 0.099$ & $\mathbf{0.3280} \pm 0.032$ \\
& Q90  & $0.2427 \pm 0.049$ & $0.2611 \pm 0.159$ & $\mathbf{0.1750} \pm 0.005$ & $0.2915 \pm 0.194$ & $0.1866 \pm 0.025$ \\
\midrule
\multirow{2}{*}{KDDCup}
& $\text{CRPS}_{\text{sum}}$ & $0.0933 \pm 0.041$ & $0.1045 \pm 0.020$ & $0.0793 \pm 0.013$ & $0.1046 \pm 0.028$ & $\mathbf{0.0578} \pm 0.001$ \\
& Q90  & $0.2244 \pm 0.012$ & $0.2338 \pm 0.009$ & $0.2226 \pm 0.009$ & $0.2278 \pm 0.010$ & $\mathbf{0.2054} \pm 0.007$\\
\midrule
\multirow{2}{*}{SP500}
& $\text{CRPS}_{\text{sum}}$ & $0.8139 \pm 0.033$ & $0.8215 \pm 0.037$ & $0.8326 \pm 0.031$ & $0.8132 \pm 0.038$ & $\mathbf{0.7958} \pm 0.018$\\
& Q90  & $0.5253 \pm 0.005$ & $0.5328 \pm 0.012$ & $0.5278 \pm 0.003$ & $0.5335 \pm 0.011$ & $\mathbf{0.5235} \pm 0.008$\\
\bottomrule
\end{tabular}
\end{adjustbox}
\end{table}

Table~\ref{tab:hybrid_results} compares the forecast performance across different hybrid formulations under a controlled environment, where DP denotes Deterministic-as-Prior, and subscripts (e.g., $_{0.3}$ or $_0$) indicate the weight $\lambda$ assigned to the point-wise deterministic loss. We use $\lambda=0.3$ to ensure a fair comparison with our proposed parameterization, and additionally consider  $\lambda=0$ to isolate the effect of the point-wise loss during training. Compared with Residual$_{0.3}$ and DP$_{0.3}$, DV-Flow achieves the lowest $\mathrm{CRPS}_{\mathrm{sum}}$ on all six datasets and the lowest $Q90$ on five of the six datasets. Because the three formulations use identical architectures, preprocessing, priors, auxiliary losses, optimization settings, and training budgets, these improvements isolate the benefit of the proposed dual-velocity parameterization. Additional controlled experiments using standard normalization and an independent Gaussian prior are reported in Appendix~\ref{app:controlled_exp}. Under these simplified settings, DV-Flow continues to outperform the residual-based and deterministic-as-prior formulations, providing further evidence that its advantage arises from the dual-velocity parameterization rather than the choice of normalization or prior.

Among the remaining hybrid variants, training with $\lambda=0.3$ performs more competitively than $\lambda=0$, further demonstrating that incorporating point-wise supervision for the deterministic predictor improves forecasting performance. 

More broadly, the results reveal that end-to-end coupling alone is insufficient for strong performance; without point-wise deterministic anchoring or gradient control, previous hybrid formulations do not have an advantage over DV-Flow. This reinforces our theoretical analysis: stable hybrid modeling requires both a fixed probability path and a properly aligned deterministic component. DV-Flow uniquely satisfies these conditions, explaining its consistent dominance across datasets and metrics.


\paragraph{Velocity error.} In this section, we investigate how effectively each hybrid formulation fits the velocity field. We have shown in Section~\ref{sec: comparable-loss} that the velocity error can be directly compared across hybrid methods. We therefore analyze how different formulations affect the velocity error. 
\begin{figure*}[ht]
\begin{center}
\centerline{\includegraphics[width=0.9\textwidth]{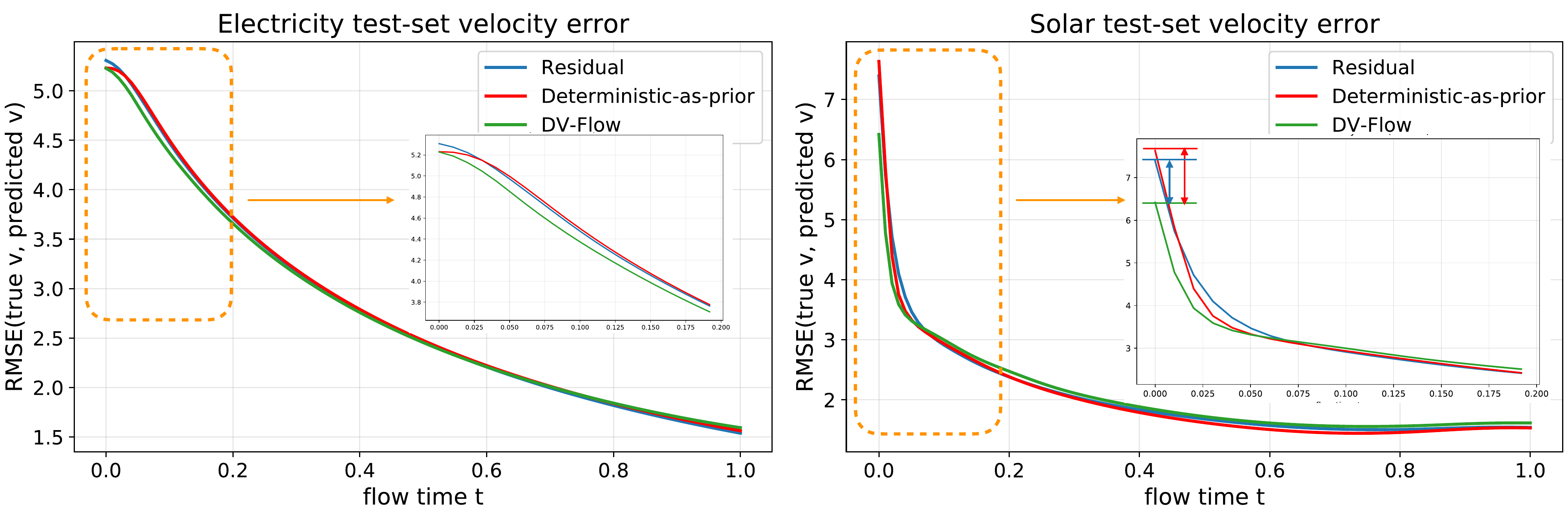}}
\caption{\textbf{Velocity Error Comparison.} We compare the RMSE of velocity predictions against ground truth on the test set for DV-Flow, residual-based, and deterministic-as-prior. While all three formulations converge as $\tau \rightarrow 1$, DV-Flow (green) starts with the lowest velocity error when $\tau \rightarrow 0$, demonstrating our formulation has an advantage when $\tau$ is small.}
\label{fig:velocity_error}
\end{center}
\end{figure*}
We evaluate the velocity prediction error using the test sets of Electricity and Solar. We compute the learning objective from Eq~\ref{eq:coflow}. As shown in Figure~\ref{fig:velocity_error}, DV-Flow achieves the lowest velocity error when the flow time $\tau$ is small, demonstrating the effectiveness of the proposed hybrid formulation. Errors in the early stages of ODE integration accumulate over time, making accurate velocity estimation near the noise distribution especially important for maintaining trajectory quality throughout the generation process. By reducing initial trajectory deviations, the proposed DV-Flow method improves the fidelity of the final generated samples. Velocity-error plots for the remaining datasets are provided in Appendix~\ref{app:velocity_error} and exhibit a similar pattern. In particular, DV-Flow consistently achieves lower velocity error near $\tau\rightarrow 0$ than the residual-based and deterministic-as-prior formulations, providing further evidence that the proposed dual-velocity parameterization more accurately estimates the velocity field during the early stages of the generative trajectory.

\paragraph{Training stability.} 
In this section, we empirically test the theoretical claim that state-dependent coupling in fully joint hybrid formulations leads to unstable training dynamics. 

\begin{figure}[htbp]
\centering
\includegraphics[width=0.45\linewidth]{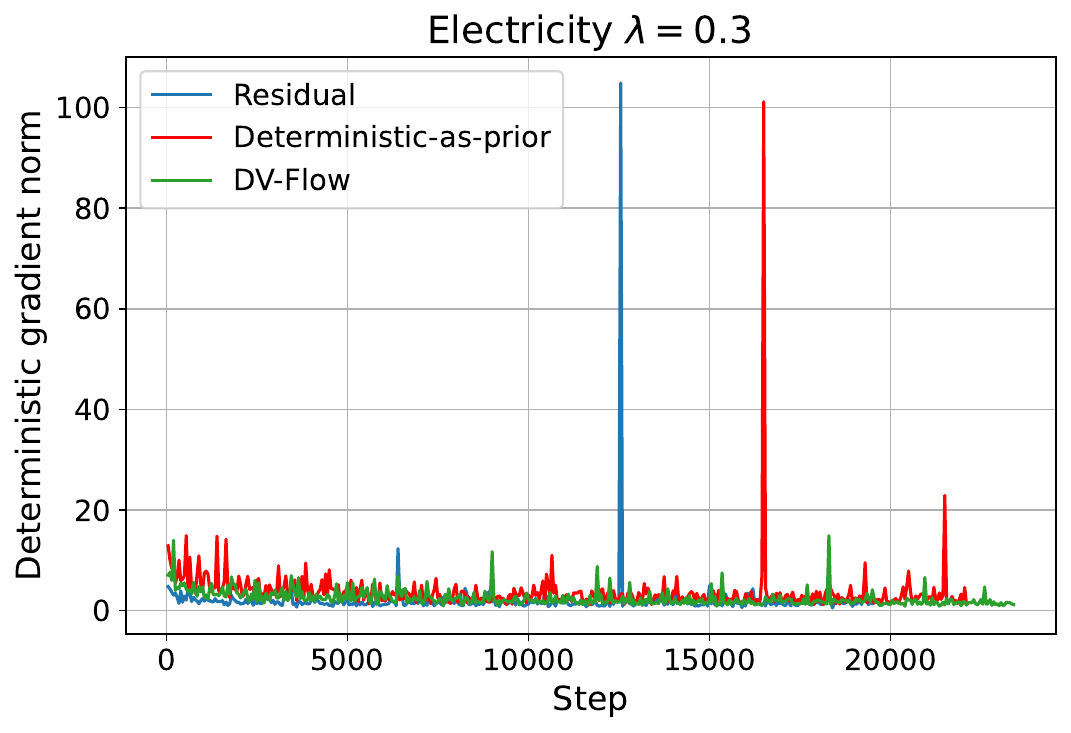}
\includegraphics[width=0.45\linewidth]{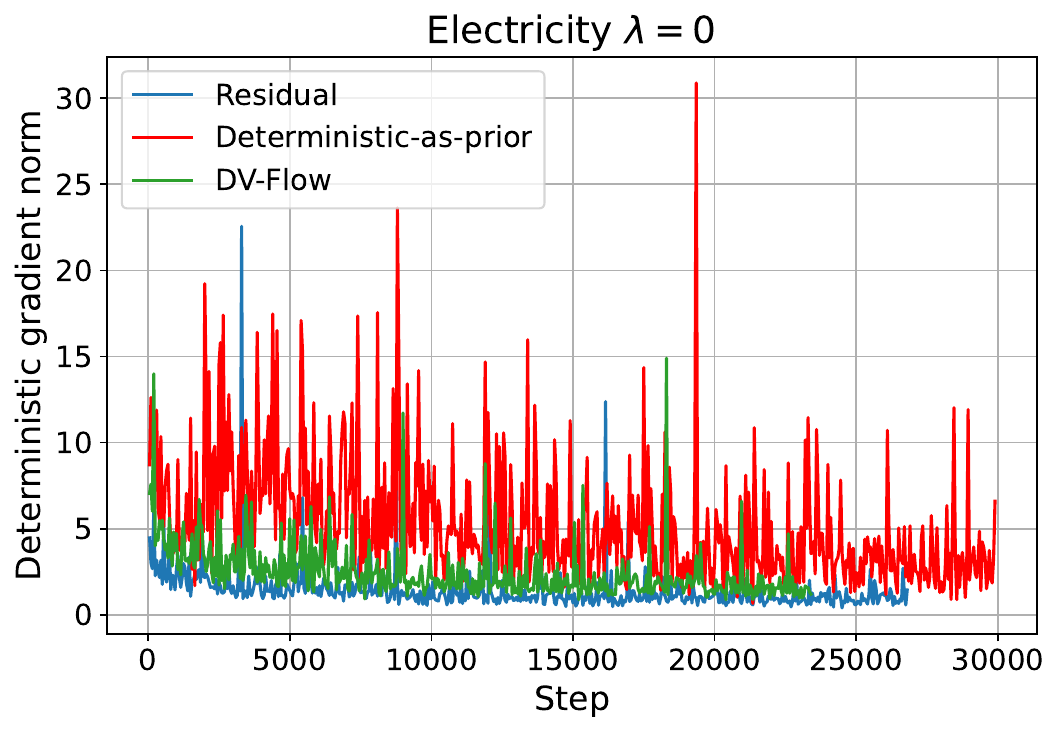}
\caption{\textbf{Gradient norm.}  On the left, we compare DV-Flow, the residual-based and deterministic-as-prior formulations with $\lambda=0.3$. On the right, we repeat the comparison with $\lambda=0$. In both settings, residual-based and deterministic-as-prior training exhibits large gradient spikes, whereas DV-Flow remains stable without large spikes.}
\label{fig:gradnorm}
\end{figure}



We evaluate training stability by measuring the trace of the gradient norm during training. In Figure~\ref{fig:gradnorm}, we compare DV-Flow, residual-based, and deterministic-as-prior variants for both $\lambda=0.3$ (left) and $\lambda=0$ (right). In both cases, residual-based and deterministic-as-prior exhibit large spikes in the gradient norm, whereas DV-Flow does not show such spikes. These sharp gradient spikes indicate optimization instability and support our hypothesis that nested coupling creates a deep computation graph, which leads to ill-conditioned gradients and unstable training. In contrast, the DV-Flow formulation avoids deep sequential coupling by placing the global and local velocities in parallel, resulting in stable gradient behavior without large spikes. Gradient-norm plots for the remaining datasets are provided in Appendix~\ref{app:grad_norm} and exhibit a similar pattern. Across these datasets, the residual-based and deterministic-as-prior formulations show pronounced gradient spikes, whereas DV-Flow maintains substantially more stable gradient norms throughout training.



\begin{figure}[t]
  \centering
  \begin{tabular}{ll}
   \raisebox{1mm}{
   \scalebox{0.9}{
   \input{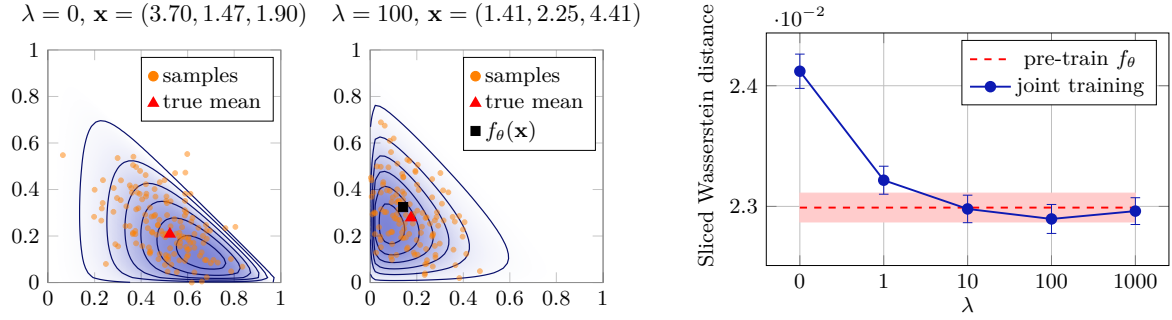} 
   }}
   \scalebox{0.9}{
\definecolor{darkblue}{rgb}{0.05, 0.1, 0.7}
\begin{tikzpicture}
\begin{axis}[
  width=7.5cm, height=5cm,
  xlabel={$\lambda$},
  ylabel={Sliced Wasserstein distance},
  xlabel style={font=\small, yshift=1mm},
  ylabel style={font=\small, xshift=2mm},
  tick label style={font=\small},
  legend style={font=\small, at={(0.97,0.97)}, anchor=north east},
  xtick={0,1,2,3,4},
  xticklabels={0,1,10,100,1000},
  ymin=0.02251,
  ytick={0.02300,0.02400},
  grid=both,
  grid style={line width=0.2pt, draw=gray!30},
  major grid style={line width=0.4pt, draw=gray!50},
  clip=false,
]

\addplot [red!20, fill=red!20, forget plot]
  coordinates {(0,0.022871) (4,0.022871)
               (4,0.023110) (0,0.023110)}
  -- cycle;

\addplot [red, dashed, thick, forget plot]
  coordinates {(0,0.022990) (4,0.022990)};
\addlegendimage{red, dashed, thick}
\addlegendentry{pre-train $f_\theta$}

\addplot [darkblue, thick, mark=*, mark size=2pt, error bars/.cd,
          y dir=both, y explicit]
  coordinates {(0,0.024120) +- (0,0.000142) (1,0.023217) +- (0,0.000116) (2,0.022979) +- (0,0.000114) (3,0.022896) +- (0,0.000120) (4,0.022961) +- (0,0.000111)};
\addlegendentry{joint training}

\end{axis}
\end{tikzpicture}
}
  \end{tabular}
  \caption{Investigation of the effect of point-wise loss on a 
  toy dataset generated from a Dirichlet distribtuion.}
  \label{fig:toy-example1}
\end{figure}

\subsection{Investigation of the effect of the point-wise loss}
\label{sec:pointwise_loss}

We now study the effect of the point-wise loss across different regularization strengths and loss functions.

\paragraph{The effect of point-wise loss in a toy example.} We first study the proposed method on a toy dataset. We generate data from a Dirichlet distribution, where the input $\bx = (x_1, x_2, x_3)$ has three entries. For each sample, one component is selected uniformly at random and drawn from $[3, 4]$, while the remaining two components are independently drawn from $[1, 2]$. We then sample a probability vector from a Dirichlet distribution parameterized by $\bx$, and define the target $\by = (y_1, y_2)$ as the first two entries of the sampled probability vector. The conditional distribution of $\by$ is fully determined by $\bx$ and is known analytically.


We train our model using 1,024 $(\bx, \by)$ training samples. After training, we generate samples of $\by$ conditioned on each $\bx$ and compare the resulting distributions with the ground-truth distributions using the sliced Wasserstein distance, where a smaller value indicates better performance. Figure~\ref{fig:toy-example1} presents the results for the first toy dataset, while an additional example is provided in Figure~\ref{fig:toy-example2} in the appendix. In the first example, $\lambda = 100$ achieves the best results, although its improvement over the variant with a pre-trained $f_\theta$ is relatively small. In contrast, for the second toy dataset in Appendix~\ref{app:more-results-on-toy-dataset}, which poses a more challenging learning problem, setting $\lambda = 0$ yields significantly better performance than using a pre-trained $f_\theta$. Overall, our experiments do not reveal a universally optimal choice of $\lambda$. Treating $\lambda$ as a tunable hyperparameter therefore appears to be a promising direction for further improving predictive performance.


\paragraph{Strength of regularization from the point-wise loss.} We compare three  strategies: no-deterministic component, pre-train the deterministic component ($\lambda=+\infty$), and joint training without point-wise loss ($\lambda=0$). For the no-deterministic configuration, we additionally evaluate a doubled-size neural network with a parameter count comparable to that of DV-Flow, in order to determine whether DV-Flow’s performance gains arise simply from increased model capacity.
As shown in Table~\ref{tab:ablation_results}, DV-Flow outperforms all three formulations in terms of $\text{CRPS}_{\text{sum}}$ across all datasets, and achieves the best $Q90$ on four out of five datasets, ranking second on the remaining one. This demonstrates that each component contributes to performance: the deterministic model improves the conditioning of context, joint training enables effective interaction between components, and the point-wise loss ensures proper alignment of the deterministic predictor.


\begin{table}[t]
\centering
\caption{Test set performance across $\lambda$ settings. $\downarrow$}
\label{tab:ablation_results}
\scriptsize
\begin{adjustbox}{width=0.95\textwidth}
\begin{tabular}{llrrrrr}
\toprule
Method & Metric 
& Electricity 
& Exchange
& Traffic 
& Solar 
& KDDCUP 
\\
\midrule
\multirow{2}{*}{No-deterministic}
& $\text{CRPS}_{\text{sum}}$ & $0.0218 \pm 0.001$ & $0.0063 \pm 0.002$ & $0.0747 \pm 0.039$ & $0.3964 \pm 0.052$ & $0.1052 \pm 0.053$\\
& Q90  & $0.0349 \pm 0.000$ & $0.0051 \pm 0.001$ & $0.1247 \pm 0.025$ & $\mathbf{0.1633} \pm 0.010$ & $0.2261 \pm 0.013$ \\
\midrule
\multirow{2}{*}{No-deterministic 2X}
& $\text{CRPS}_{\text{sum}}$ & $0.0291 \pm 0.006$ & $0.0050 \pm 0.000$ & $0.0492 \pm 0.007$ & $0.4007 \pm 0.023$ & $0.1213 \pm 0.022$  \\
& Q90  & $0.0411 \pm 0.005$ & $0.0052 \pm 0.001$ & $0.1070 \pm 0.002$ & $0.2412 \pm 0.092$ & $0.2312 \pm 0.004$ \\
\midrule
\multirow{2}{*}{Pre-train}
& $\text{CRPS}_{\text{sum}}$ & $0.0235 \pm 0.002$ & $0.0053 \pm 0.000$ & $0.0672 \pm 0.010$ & $0.5119 \pm 0.007$ & $0.2002 \pm 0.137$  \\
& Q90  & $0.0418 \pm 0.004$ & $0.0042 \pm 0.000$ & $0.1434 \pm 0.006$ & $0.3317 \pm 0.035$ & $0.2687 \pm 0.022$ \\
\midrule
\multirow{2}{*}{$\lambda=0$}
& $\text{CRPS}_{\text{sum}}$ & $0.0301 \pm 0.000$ & $0.0054 \pm 0.001$ & $0.0426 \pm 0.008$ & $0.4006 \pm 0.102$ & $0.1213 \pm 0.022$  \\
& Q90  & $0.0378 \pm 0.002$ & $0.0043 \pm 0.000$ & $0.1063 \pm 0.004$ & $0.2590 \pm 0.134$ & $0.2312 \pm 0.004$ \\
\midrule
\multirow{2}{*}{DV-Flow/MSE}
& $\text{CRPS}_{\text{sum}}$ & $0.0203 \pm 0.002$ & $0.0047 \pm 0.000$ & $\mathbf{0.0372} \pm 0.004$ & $0.3318 \pm 0.029$ & $0.1028 \pm 0.014$  \\
& Q90  & $\mathbf{0.0343} \pm 0.002$ & $\mathbf{0.0039} \pm 0.000$ & $0.1057 \pm 0.002$ & $0.1963 \pm 0.017$ & $0.2230 \pm 0.007$ \\
\midrule
\multirow{2}{*}{DV-Flow/MAE}
& $\text{CRPS}_{\text{sum}}$ & $\mathbf{0.0197} \pm 0.000$ & $\mathbf{0.0046} \pm 0.000$ & $0.0386 \pm 0.005$ & $\mathbf{0.3280} \pm 0.032$ & $\mathbf{0.0578} \pm 0.001$  \\
& Q90  & $0.0348 \pm 0.002$ & $0.0041 \pm 0.000$ & $\mathbf{0.1054} \pm 0.003$ & $0.1866 \pm 0.025$ & $\mathbf{0.2054} \pm 0.007$ \\
\bottomrule
\end{tabular}
\end{adjustbox}
\end{table}
\begin{wrapfigure}{r}{0.42\textwidth}
    \centering
    \includegraphics[width=\linewidth]{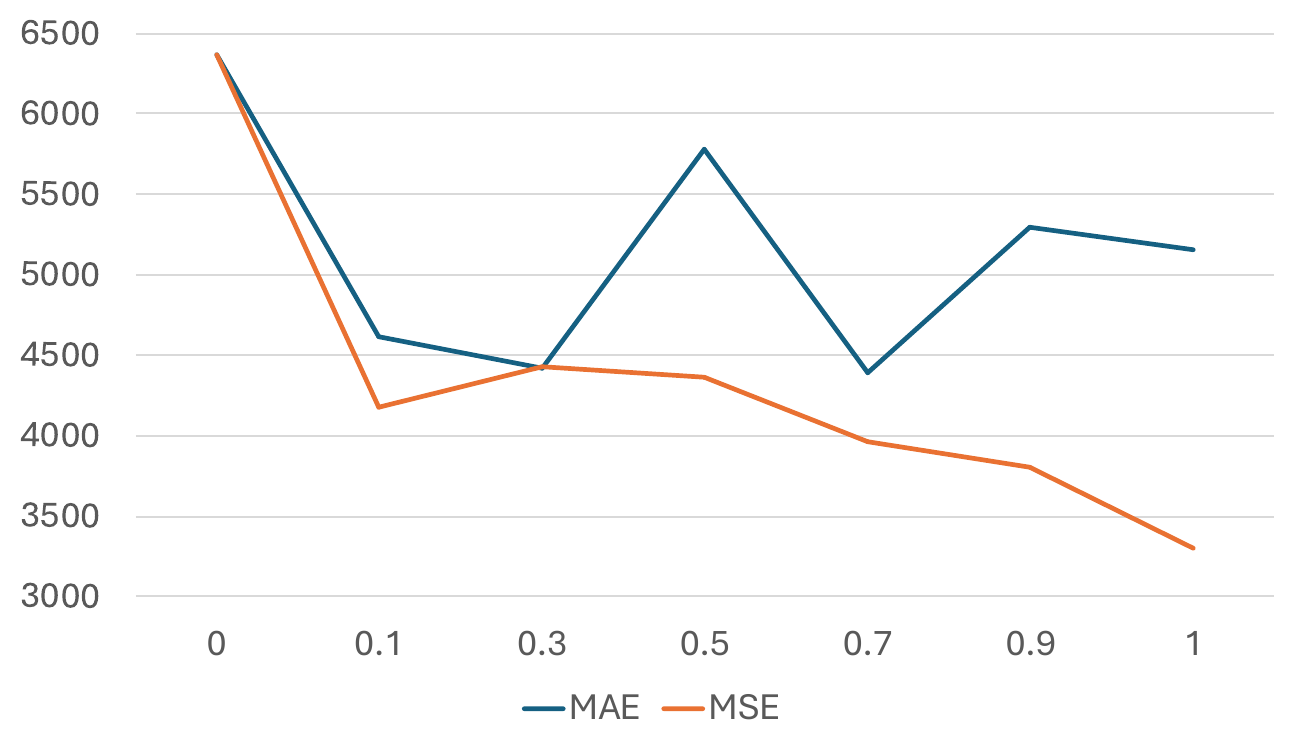}
    \caption{$\ell_2$ distance between the prediction mean and median across various $\lambda$ levels.}
    \label{fig:msevsmae}
    \vspace{-15pt}
\end{wrapfigure}
\paragraph{Choice of point-wise loss.} We use MAE as the point-wise loss term in the objective function. As shown in \cite{baringhaus2004new, szekely2005hierarchical, gneiting2007strictly}, CRPS can be written as the following, where $Z, Z' \sim F$  are independent samples and $x$ is the observation:
\begin{equation} 
\label{eq:crps}
\text{CRPS} = \mathbb{E}_F|Z-x| - \frac{1}{2}\mathbb{E}_F|Z-Z'|.  
\end{equation}

It shows that MAE is closely aligned with the CRPS evaluation objective
used in probabilistic forecasting. This formulation highlights CRPS as a distributional generalization of
MAE: the first term corresponds to the expected absolute deviation from the observation, while the second
term penalizes excessive dispersion. In the degenerate case where F collapses to a point mass, CRPS reduces
to MAE. Therefore, MAE provides a more natural auxiliary objective when the downstream goal is strong
CRPS performance. Under this view, the deterministic model supplies a median-centered baseline forecast,
while the flow component focuses on refining the predictive distribution around this baseline. Table~\ref{tab:ablation_results} shows that training with MAE loss has an overall better $\text{CRPS}_{\text{sum}}$ than training with MSE loss.

Furthermore, MAE supervision preserves skewness more effectively than MSE supervision on skewed data. We demonstrate it with the Electricity dataset, which exhibits large outliers during peak usage periods. As shown in Figure~\ref{fig:msevsmae}, the $\ell_2$ distance between the prediction mean and median decreases as $\lambda$ increases under MSE training, while remaining relatively stable under MAE training.

\subsection{Ablation}

\begin{table}[t]
\centering
\caption{Ablation study on dual velocity $\downarrow$}
\label{tab:ablation_dv_results}
\scriptsize
\begin{adjustbox}{width=0.95\textwidth}
\begin{tabular}{llrrrrr}
\toprule
Method & Metric 
& Electricity 
& Exchange
& Traffic 
& KDDCUP 
& SP500
\\
\midrule
\multirow{2}{*}{vanilla FM}
& $\text{CRPS}_{\text{sum}}$ & $0.0257 \pm 0.004$ & $0.0065 \pm 0.000$ & $0.0517 \pm 0.005$ & $0.2426 \pm 0.089$ & $0.8884 \pm 0.051$\\
& Q90  & $\mathbf{0.0344} \pm 0.001$ & $0.0059 \pm 0.000$ & $0.1080 \pm 0.001$ & $0.3461 \pm 0.059$ & $0.5246 \pm 0.005$ \\
\midrule
\multirow{2}{*}{vanilla FM + SATN}
& $\text{CRPS}_{\text{sum}}$ & $0.0259 \pm 0.003$ & $0.0056 \pm 0.000$ & $0.0546 \pm 0.006$ & $0.1709 \pm 0.092$ & $0.8838 \pm 0.052$  \\
& Q90  & $0.0396 \pm 0.005$ & $0.0046 \pm 0.000$ & $0.1204 \pm 0.015$ & $0.2683 \pm 0.051$ & $0.5258 \pm 0.003$ \\
\midrule
\multirow{2}{*}{vanilla FM + SATN + GP}
& $\text{CRPS}_{\text{sum}}$ & $0.0247 \pm 0.003$ & $0.0053 \pm 0.000$ & $0.0461 \pm 0.011$ & $0.1407 \pm 0.020$ & $0.8560 \pm 0.016$  \\
& Q90  & $0.0368 \pm 0.002$ & $0.0043 \pm 0.000$ & $0.1065 \pm 0.002$ & $0.2341 \pm 0.001$ & $0.5382 \pm 0.015$ \\
\midrule
\multirow{2}{*}{DV-Flow}
& $\text{CRPS}_{\text{sum}}$ & $\mathbf{0.0197} \pm 0.000$ & $\mathbf{0.0046} \pm 0.000$ & $\mathbf{0.0386} \pm 0.005$ & $\mathbf{0.0578} \pm 0.001$ & $\mathbf{0.7958} \pm 0.018$ \\
& Q90  & $0.0348 \pm 0.002$ & $\mathbf{0.0041} \pm 0.000$ & $\mathbf{0.1054} \pm 0.003$  & $\mathbf{0.2054} \pm 0.007$ & $\mathbf{0.5235} \pm 0.008$ \\
\bottomrule
\end{tabular}
\end{adjustbox}
\end{table}

To distinguish the contribution of the proposed dual-velocity
parameterization from those of SATN and the Gaussian-process prior, we
conduct a controlled incremental ablation. Starting from vanilla flow
matching, we progressively introduce SATN, the Gaussian-process prior, and
the dual-velocity parameterization. All configurations use the same
backbone architecture, optimization settings, and training budget.
As shown in Table~\ref{tab:ablation_dv_results}, DV-Flow achieves the
lowest $\mathrm{CRPS}_{\mathrm{sum}}$ across all datasets, often by a
substantial margin, and the lowest $Q90$ on four of the five datasets.
Most importantly, the comparison between the complete DV-Flow model and  vanilla flow matching with both SATN and the Gaussian-process prior differs
only in the velocity parameterization. Therefore, the improvement between
these configurations directly isolates the contribution of the
proposed dual-velocity design, rather than attributing it to SATN or the
Gaussian-process prior.

\section{Limitation}
DV-Flow moderately increases the number of learnable parameters and the runtime during both training and forecasting. Adding the deterministic predictor increases the total parameter count of DV-Flow to approximately $1.3\times$ that of the corresponding single-network model. The training time per epoch is approximately $1.4\times$ that of the single-network model, $2.2\times$ that of representative probabilistic-only baselines, and $1.4\times$ that of other hybrid baselines. The peak GPU memory usage is approximately $1.14\times$ that of the single-network model and approximately the same as that of the hybrid baselines. DV-Flow requires approximately $1.1\times$ the inference time of flow-based single-network baseline, and $2.1\times$ that of the diffusion-based hybrid baselines. We use 100 function evaluations across all models for the inference time analysis. We use 100 function evaluations across all models for the inference time analysis. 


\section{Conclusion}
We introduce DV-Flow, a hybrid flow-matching framework that incorporates a deterministic model as an additional velocity component, rather than treating it as the endpoint of the transport path. This design avoids deep computational graphs and preserves a fixed probability path, resulting in more stable and efficient end-to-end training. Empirically, DV-Flow consistently outperforms baseline methods across multiple benchmark datasets, improving both predictive performance and training stability in fully joint settings. The proposed dual-velocity parameterization may also extend to broader conditional-generation domains, which we leave for future work.



\section*{Impact Statement}

This work aims to improve the reliability and stability of probabilistic time series forecasting, which may benefit decision-making in applications such as weather, energy, and finance, while sharing the same considerations regarding data bias and responsible deployment as other forecasting models.

\nocite{*}

\bibliography{references}
\bibliographystyle{tmlr}

\newpage
\appendix
\appendix

\section{Derivation of distribution learning with linear assumption}
\label{sec:appendixA}

\subsection{Theorem proof}
\begin{proof}
Consider the DV-Flow ODE
\begin{equation}
    \frac{d\by^\tau}{d\tau}
    =
    f_\theta(\bx)+g_\phi(\by^\tau,\tau,\bx),
    \qquad
    \by^0=\bepsilon .
\end{equation}
Define the shifted residual coordinate
\begin{equation}
    \br^\tau = \by^\tau - \tau f_\theta(\bx).
\end{equation}
Since $f_\theta(\bx)$ is independent of $\tau$, differentiating both sides gives
\begin{equation}
    \frac{d\br^\tau}{d\tau}
    =
    \frac{d\by^\tau}{d\tau}
    -
    f_\theta(\bx).
\end{equation}
Substituting the DV-Flow ODE, we obtain
\begin{align}
    \frac{d\br^\tau}{d\tau}
    &=
    f_\theta(\bx)+g_\phi(\by^\tau,\tau,\bx)
    -
    f_\theta(\bx) \\
    &=
    g_\phi(\by^\tau,\tau,\bx).
\end{align}
Using $\by^\tau=\br^\tau+\tau f_\theta(\bx)$, this becomes
\begin{equation}
    \frac{d\br^\tau}{d\tau}
    =
    g_\phi(\br^\tau+\tau f_\theta(\bx),\tau,\bx).
\end{equation}
Let
\begin{equation}
    g'(\br^\tau,\tau)
    =
    g_\phi(\br^\tau+\tau f_\theta(\bx),\tau,\bx).
\end{equation}
Then the residual coordinate follows the ODE
\begin{equation}
    \frac{d\br^\tau}{d\tau}
    =
    g'(\br^\tau,\tau),
    \qquad
    \br^0=\by^0=\bepsilon.
\end{equation}
At $\tau=1$, we have
\begin{equation}
    \br^1
    =
    \by^1 - f_\theta(\bx).
\end{equation}
Therefore, if the DV-Flow ODE transports $\bepsilon$ to samples
$\by^1\sim p(\by\mid \bx)$, then the transformed ODE defined by $g'$
transports $\bepsilon$ to samples
\begin{equation}
    \br^1 = \by^1 - f_\theta(\bx).
\end{equation}
Equivalently, the terminal distribution of the residual-coordinate ODE is
the distribution of $\by-f_\theta(\bx)$.

In particular, if
\begin{equation}
    f_\theta(\bx)=\mathbb{E}[\by\mid \bx],
\end{equation}
then
\begin{equation}
    \mathbb{E}[\br^1\mid \bx]
    =
    \mathbb{E}[\by-f_\theta(\bx)\mid \bx]
    =
    \mathbb{E}[\by\mid \bx]-f_\theta(\bx)
    =
    0.
\end{equation}
Thus the ODE defined by $g'$ pushes Gaussian noise to a zero-mean residual
distribution.

Alternatively, suppose that $f_\theta(\bx)$ is a coordinate-wise
conditional median of $\by$ given $\bx$, i.e.,
\begin{equation}
    f_\theta(\bx)
    =
    \operatorname{Med}(\by\mid\bx).
\end{equation}
By the translation equivariance of the median,
\begin{equation}
\begin{aligned}
    \operatorname{Med}(\br^1\mid\bx)
    &=
    \operatorname{Med}\bigl(\by-f_\theta(\bx)\mid\bx\bigr) \\
    &=
    \operatorname{Med}(\by\mid\bx)-f_\theta(\bx)
    =
    0.
\end{aligned}
\end{equation}
Thus, the ODE defined by $g'$ pushes Gaussian noise to a residual
distribution for which zero is a coordinate-wise conditional median.
\end{proof}

\subsection{General linear form}

Assume the prior is a standard Gaussian distribution $\by^0 \sim \mathcal{N}(0,I)$, and the local velocity has a linear form: $g_{\phi}=A(\tau)\by^{\tau}+a(\tau)$. The ODE becomes: $\Dot{\by}^{\tau} = A(\tau)\by^{\tau}+a(\tau)+f_{\theta}(\bx)$. Consider the homogeneous part 
\[
\Dot{\by}^{\tau} = A(\tau)\by^{\tau}
\]
The solution is $\by^{\tau} = \Phi(\tau,s)\by^s$, where $\frac{\partial \Phi(\tau,s)}{\partial \tau} = A(\tau)\Phi(\tau,s)$, and $\Phi(s,s)=I$. Now make an ansatz for some vector $c(\tau)$:  
\begin{align}
\label{eq:ansatz}
    \by^{\tau}=\Phi(\tau,0)c(\tau)
\end{align}
If we differentiate the ansatz:
\begin{align}
\begin{split}
    \Dot{\by}^{\tau} &= \frac{\partial \Phi(\tau,0)}{\partial \tau}c(\tau) +  \Phi(\tau,0) \Dot{c(\tau)} \\
    &= A(\tau)\Phi(\tau,0)c(\tau) +\Phi(\tau,0)  \Dot{c(\tau)} \\
    &= A(\tau)\bx^{\tau} + \Phi(\tau,0)  \Dot{c(\tau)}
\end{split}
\end{align}
Compare with the original ODE $\Dot{\by}^{\tau} = A(\tau)\by^{\tau}+a(\tau)+f_{\theta}(\bx)$, we get $\Phi(\tau,0)  \Dot{c(\tau)} = a(\tau)+f_{\theta}(\bx)$. Therefore, $\Dot{c(\tau)} = \Phi(\tau,0)^{-1}(a(\tau)+f_{\theta}(\bx)) = \Phi(0,\tau)(a(\tau)+f_{\theta}(\bx))$. Integrate for $c(\tau)$:
\begin{align}
    c(\tau) = c(0)+\int_0^{\tau} \Phi(0,s)(a(s)+f_{\theta}(\bx)) ds
\end{align}
At $\tau=0$, $\by^{\tau} = \Phi(0,0)c(0)=Ic(0)=\by^0$, therefore, $c(0)=\by^0$. Back substitute Equation~\ref{eq:ansatz}:
\begin{align}
\label{eq:xtau}
\begin{split}
    \by^{\tau}&=\Phi(\tau,0)c(\tau) \\
    &=\Phi(\tau,0)[\by^0+\int_0^{\tau} \Phi(0,s)(a(s)+f_{\theta}(\bx)) ds] \\
    &= \Phi(\tau,0)\by^0 + \int_0^{\tau} \Phi(\tau,s)(a(s)+f_{\theta}(\bx)) ds
\end{split}
\end{align}
For better understanding of $\Phi(\tau,s)$, according to the Peano-Baker series \cite{baake2011peano}:
\begin{align}
\label{eq:bpseries}
    \begin{split}
        \Phi(\tau,s) &= I + \int_s^{\tau} A(\varsigma)\Phi(\varsigma,s)d\varsigma \\
        &= I + \int_s^{\tau}A(\varsigma)d\varsigma + \int_s^{\tau}A(\varsigma_1)\int_s^{\varsigma_1}A(\varsigma_2)d\varsigma_1d\varsigma_2 + \dots \\
        &= I + \sum_{n=1}^{\infty} \mathcal{I}_n(\tau)
    \end{split}
\end{align}
Where
\[
\mathcal{I}_n(\tau) = \int_s^{\tau}A(\varsigma_1)\int_s^{\varsigma_1}A(\varsigma_2)\dots \int_s^{\varsigma_{n-1}}A(\varsigma_n)d\varsigma_1d\varsigma_2\dots d\varsigma_n
\]

According to Equation~\ref{eq:xtau}, $\by^1 = \Phi(1,0)\by^0 + \int_0^{1} \Phi(1,s)(a(s)+f_{\theta}(\bx)) ds$.
\begin{equation}
\label{eq:final_dist}
    \begin{split}
        \bmu^1 &= \mathbb{E}[\by^1] \\
        &= \mathbb{E}[\Phi(1,0)\by^0 + \int_0^{1} \Phi(1,s)(a(s)+f_{\theta}(\bx)) ds] \\
        &= \Phi(1,0)\mathbb{E}[\by^0] + \int_0^{1} \Phi(1,s)(a(s)+f_{\theta}(\bx)) ds \\
        &= \int_0^{1} \Phi(1,s)(a(s)+f_{\theta}(\bx)) ds \\
        \bSigma^1 &= Var(\by^1) = Var(\Phi(1,0)\by^0 + \int_0^{1} \Phi(1,s)(a(s)+f_{\theta}(\bx)) ds) \\
        &= Var(\Phi(1,0)\by^0) = \Phi(1,0)I\Phi(1,0)^T
    \end{split}
\end{equation}

\subsection{Special Case of Constant A}
For the special case where $A(\tau) \equiv A$ is constant and $a(\tau) \equiv 0$, solving the ODE $\frac{\partial \Phi(\tau,s)}{\partial \tau} = A\Phi(\tau,s)$ with initial condition $\Phi(s,s)=I$, we get $\Phi(\tau,s) = e^{A(\tau-s)}$. Therefore, $\Phi(1,s) = e^{A(1-s)}$, plug in Equation~\ref{eq:final_dist} we get:
\begin{align}
\begin{split}
\bmu_1 &= \int_0^1 \Phi(1,s)f_{\theta}(\bx)ds \\
&= f_{\theta}(\bx)\int_0^1e^{A(1-s)}ds \\
&= f_{\theta}(\bx)[A^{-1}e^{A(1-s)}]|_0^1 \\
&= f_{\theta}(\bx)[-A^{-1} - (-A^{-1})e^A] \\
&= A^{-1}(e^A-I)f_{\theta}(\bx)
\end{split}
\end{align}

Here, the deterministic network $f_{\theta}(\bx)$ contributes only to the mean, while the covariance is entirely determined by the local velocity through $A$. In particular, when $A$ has negative eigenvalues, the transformation $e^A$ contracts the space, leading to a small covariance $\bSigma^1$, where the rollouts of the prediction are centered around the mean.

\section{Algorithms}

\begin{minipage}{0.49\linewidth}
\begin{algorithm}[H]\small  
    \caption{Training}
  \label{alg:training}
  \begin{algorithmic}[1]
    \Require data dist.\ $p^*$; prior $p_0=\mathcal{N}(\mathbf{0},\bK(t,t'))$;  noise $\sigma_{\min}$;  $g_{\phi}$; $f_{\theta}$; $\lambda$
    \State $\mathcal X \sim p^*$ \Comment{full dataset}
    \State $\bs \gets \operatorname{calc\_seasonality}(\mathcal X)$
    \State $\boldsymbol{\sigma}_t \gets \operatorname{calc\_sigma}(\mathcal X)$
    \For{$i=1,2,\dots$}
      \State $\bx, \by \sim p^*(\bx)$; \quad $\by^0 \sim p_0$
      \State $\bx \gets \operatorname{SATN}(\bx,\bs, \bsigma)$ 
      \State $\by \gets \operatorname{SATN}(\by,\bs, \bsigma)$
      \State $\tau \sim \mathcal U(0,1)$
      \State 
        $\by^\tau
        \gets
        (1-\tau)\by^0+\tau\by$
    
      \State 
        $\bu^\tau
        \gets
        \partial_\tau\by^\tau
        =
        \by-\by^0$
    
      \State $\mathcal L(\theta, \phi) \gets
        \text{Eq}~\ref{eq:coflow}$
      \State $\theta, \phi \gets \operatorname{Update}\!\big(\theta,\, \phi, \,\nabla_{\theta, \phi}\mathcal L\big)$
    \EndFor
\end{algorithmic}
\end{algorithm}
\end{minipage}
\hfill
\begin{minipage}{0.49\linewidth}
\begin{algorithm}[H]\small  
  \caption{Sampling}
  \label{alg:sampling}
  \begin{algorithmic}[1]
    \Require prior GP noise $p_0=\mathcal{N}(\mathbf{0},\bK(t,t'))$; trained local velocity net $g_{\phi}$; trained deterministic net $f_{\theta}$; seasonal cycle $\bs$ (from train set); scale $\boldsymbol{\sigma}$ (from train set), steps $N$
    \For{$i=1,2,\dots$}
      \State $\bx \sim p^*(\bx)$; \quad $\by^0 \sim p_0$
      \State $\bx \gets \operatorname{SATN}(\bx,\bs,\bsigma)$
      \State $\hat{\by} \gets f_{\theta}(\bx)$
      \State $\hat{v} \gets g_{\phi}+\hat{\by}$
      \State
        $\begin{aligned}
        \hat{\by}^{\,1}
        \leftarrow
        &\operatorname{ODESolver}\!\big(
            \hat{v},\by^0,[0,1],N=150; \\
            &\operatorname{method}=\mathrm{Dopri5},
            \operatorname{atol}=10^{-6},
            \operatorname{rtol}=10^{-3}
        \big)
        \end{aligned}$
    \Statex \Comment{Invert SATN normalization}
      \State $\by^1 \gets \boldsymbol{\sigma}_t \odot \by^1 + \bx_{t_0}$
      \State $\by^1 \gets \operatorname{add\_seasonal}(\by^1,\,\bs)$
    \EndFor
\end{algorithmic}
\end{algorithm}
\end{minipage}

\section{Baselines}
\label{uq-exp}
For all uncertainty quantification experiments, we conduct our own evaluations of the baseline models. Specifically, we run the official implementations of TimeGrad, Diffusion-TS, TMDM and $D^3U$ using their released code and default settings. For DeepAR , TFT , PatchTST , and iTransformer , we use the implementations provided by GluonTS  applying default hyperparameters, except for iTransformer . Since our model uses iTransformer as the backbone, we align the configuration and set the hidden size to $1024$ for consistency.

\section{Evaluation Metics} \label{sec:evaluation}

Following \cite{lim2021temporal},  q-Risk is defined as:
\[
q\text{-Risk} = 
\frac{
2 \sum_{i=1}^D \sum_{x_{i,t} \in \tilde{\mathcal{X}}} \sum_{t=t_0+1}^{T} 
QL(x_{i,t}, \hat{x}_{i,t}, q)
}{\sum_{i=1}^D
\sum_{x_{i,t} \in \tilde{\mathcal{X}}} \sum_{t=t_0+1}^{T} |x_{i,t}|
}, \quad QL(x, \hat{x}, q) = q(x - \hat{x})_{+} + (1 - q)(\hat{x} - x)_{+}
\]
Where $\tilde{\mathcal{X}}$ is the test data space with $M$ samples, and $q$ is the quantile. We report q-Risk at $q=0.9$, denoted as $Q90$.

\section{Additional Controlled Experiments}
\label{app:controlled_exp}

In order to study whether the advantage of the dual-velocity design depends on the simultaneous presence of SATN and the correlated GP prior, we conducted an additional controlled experiment in which both components were removed: (1) we replaced SATN with standard global normalization, using the global mean and standard deviation of each dimension calculated from the training set, and (2) we replaced the correlated GP prior with independent Gaussian noise. From Table~\ref{tab_app:hybrid_results}, we can see that the advantage of the dual-velocity design remains present in the more simplified setting.

\begin{table}[t]
\centering
\caption{Test set performance across hybrid formulations. $\downarrow$}
\label{tab_app:hybrid_results}
\scriptsize

\begin{adjustbox}{width=0.95\textwidth}
\begin{tabular}{llrrrrrrrr}
\toprule
Dataset & Metric 
& Residual$_{0.3}$ 
& DP$_{0.3}$ 
& DV-Flow\\
\midrule
\multirow{2}{*}{Electricity}
& $\text{CRPS}_{\text{sum}}$ & $0.0309 \pm 0.007$ & $0.0231 \pm 0.02$ & $\mathbf{0.0225} \pm 0.003$ \\
& Q90  & $0.0361 \pm 0.002$ & $\mathbf{0.0334} \pm 0.003$ & $0.0342 \pm 0.001$ \\
\midrule
\multirow{2}{*}{Exchange}
& $\text{CRPS}_{\text{sum}}$ & $0.0106 \pm 0.003$ & $0.0125 \pm 0.005$ & $\mathbf{0.0053} \pm 0.000$  \\
& Q90  & $0.0154 \pm 0.010$ & $0.0169 \pm 0.009$ & $\mathbf{0.0077} \pm 0.000$  \\
\midrule
\multirow{2}{*}{Traffic}
& $\text{CRPS}_{\text{sum}}$ & $0.0475 \pm 0.002$ & $0.0512 \pm 0.018$ & $\mathbf{0.0397} \pm 0.002$  \\
& Q90  & $0.1153 \pm 0.004$ & $0.1407 \pm 0.052$ & $\mathbf{0.1094} \pm 0.003$ \\
\midrule
\multirow{2}{*}{Solar}
& $\text{CRPS}_{\text{sum}}$ & $0.4156 \pm 0.015$ & $0.4045 \pm 0.149$ & $\mathbf{0.3182} \pm 0.054$\\
& Q90  & $0.2318 \pm 0.057$ & $0.1727 \pm 0.005$ & $\mathbf{0.1700} \pm 0.004$ \\
\midrule
\multirow{2}{*}{KDDCup}
& $\text{CRPS}_{\text{sum}}$ & $0.1948 \pm 0.051$ & $0.1798 \pm 0.041$ & $\mathbf{0.1718} \pm 0.037$ \\
& Q90  & $\mathbf{0.2797} \pm 0.033$ & $0.3348 \pm 0.038$ & $0.3309 \pm 0.013$ \\
\midrule
\multirow{2}{*}{SP500}
& $\text{CRPS}_{\text{sum}}$ & $0.8134 \pm 0.072$ & $0.8480 \pm 0.058$ & $\mathbf{0.7799} \pm 0.018$\\
& Q90  & $0.5184 \pm 0.001$ & $0.5232 \pm 0.004$ & $\mathbf{0.5167} \pm 0.007$ \\
\bottomrule
\end{tabular}
\end{adjustbox}
\end{table}

\section{Velocity Error on Other Datasets}
\label{app:velocity_error}
Figure~\ref{fig:velocity_error_additional} presents the velocity errors on the Traffic, Exchange Rate, KDD, and SP500 datasets. We observe a consistent pattern across all datasets except SP500: DV-Flow achieves lower velocity errors during the early stages of the trajectory. On SP500, the three variants exhibit comparable errors, although DV-Flow maintains a slight advantage throughout the entire trajectory.
\begin{figure*}[ht]
\begin{center}
\includegraphics[page=1,width=0.9\textwidth]{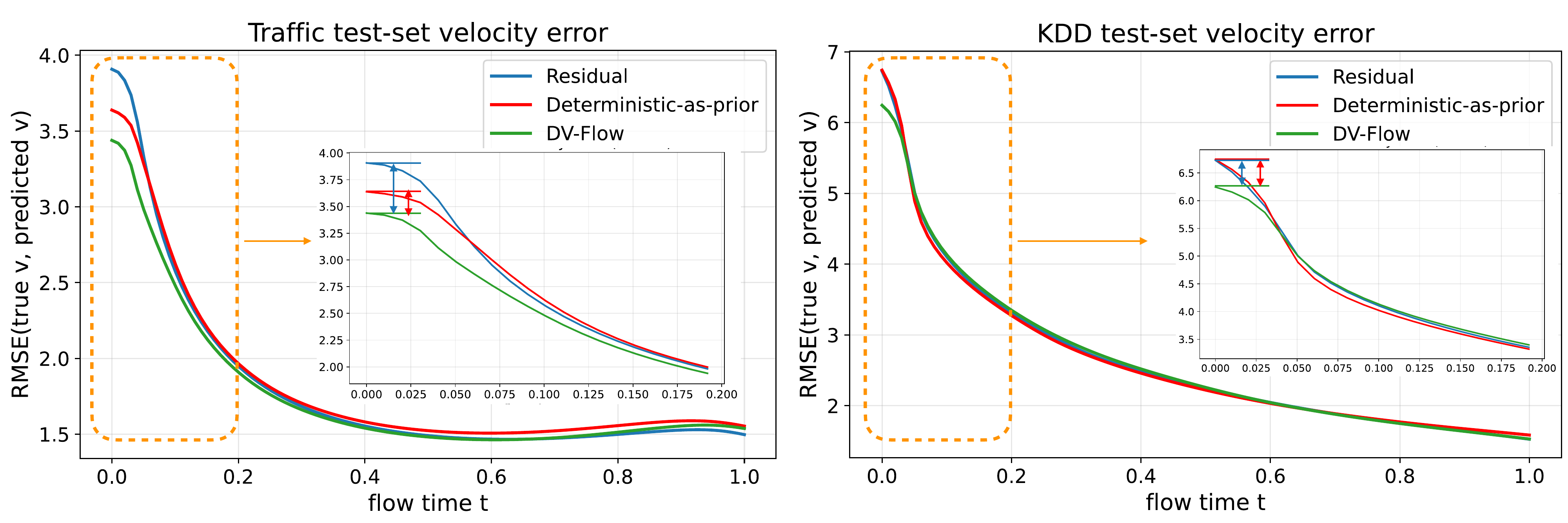}
    \vspace{0.5em}
\includegraphics[page=2,width=0.9\textwidth]{figures/velocity_error_additional.pdf}
\caption{\textbf{Velocity Error Comparison.} We compare the RMSE of velocity predictions against ground truth on the test set for DV-Flow, residual-based, and deterministic-as-prior. While all three formulations converge as $\tau \rightarrow 1$, DV-Flow (green) starts with the lowest velocity error when $\tau \rightarrow 0$, demonstrating our formulation has an advantage when $\tau$ is small.}
\label{fig:velocity_error_additional}
\end{center}
\end{figure*}

\section{Gradient Norm on Other Datasets}
\label{app:grad_norm}

Figure~\ref{fig:gradnorm_additional} presents the gradient norms on the additional datasets with $\lambda=0.3$. Across all datasets, the residual-based and deterministic-as-prior formulations exhibit large gradient spikes, whereas DV-Flow maintains stable gradients without pronounced spikes.

\begin{figure}[htbp]
\centering
\includegraphics[width=0.45\linewidth]{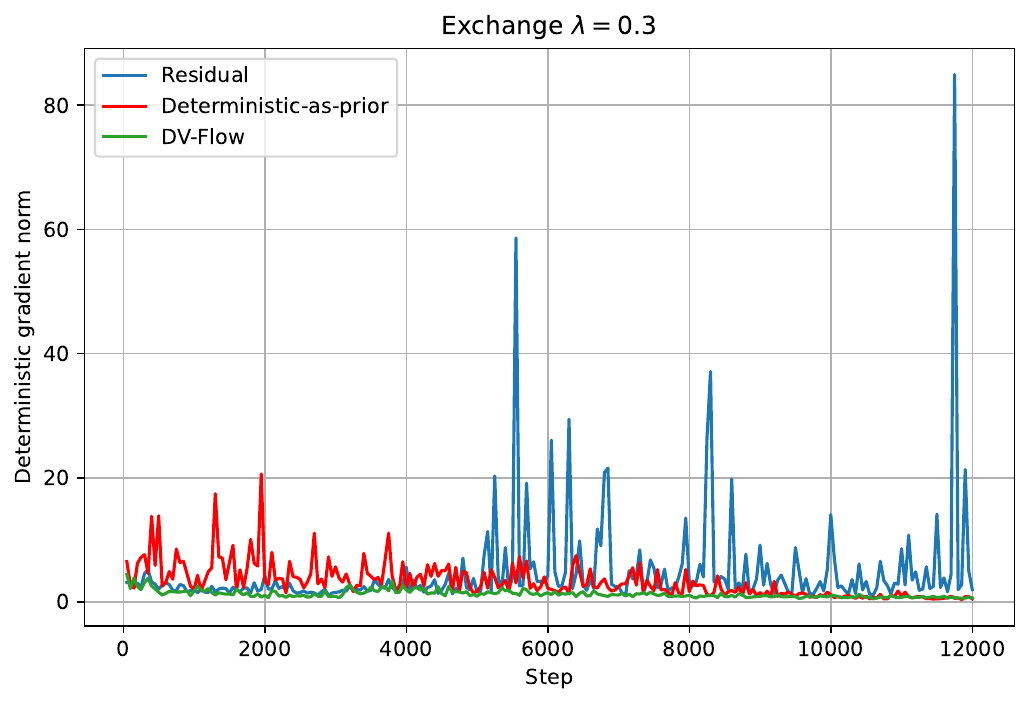}
\includegraphics[width=0.45\linewidth]{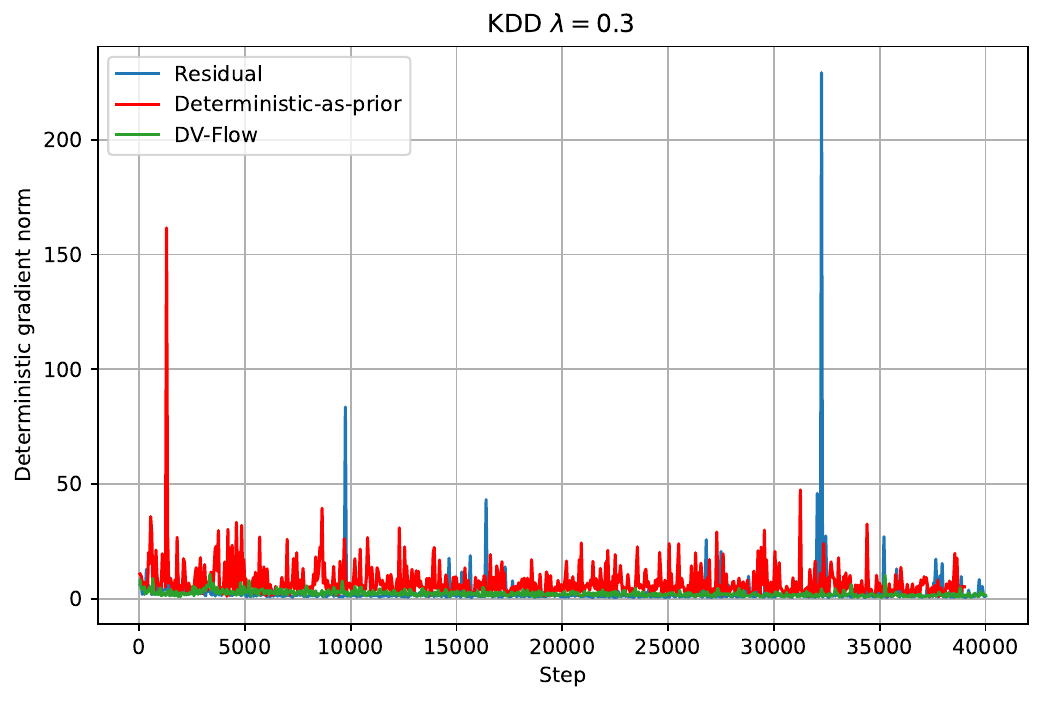}
\vspace{0.5em}
\includegraphics[width=0.45\linewidth]{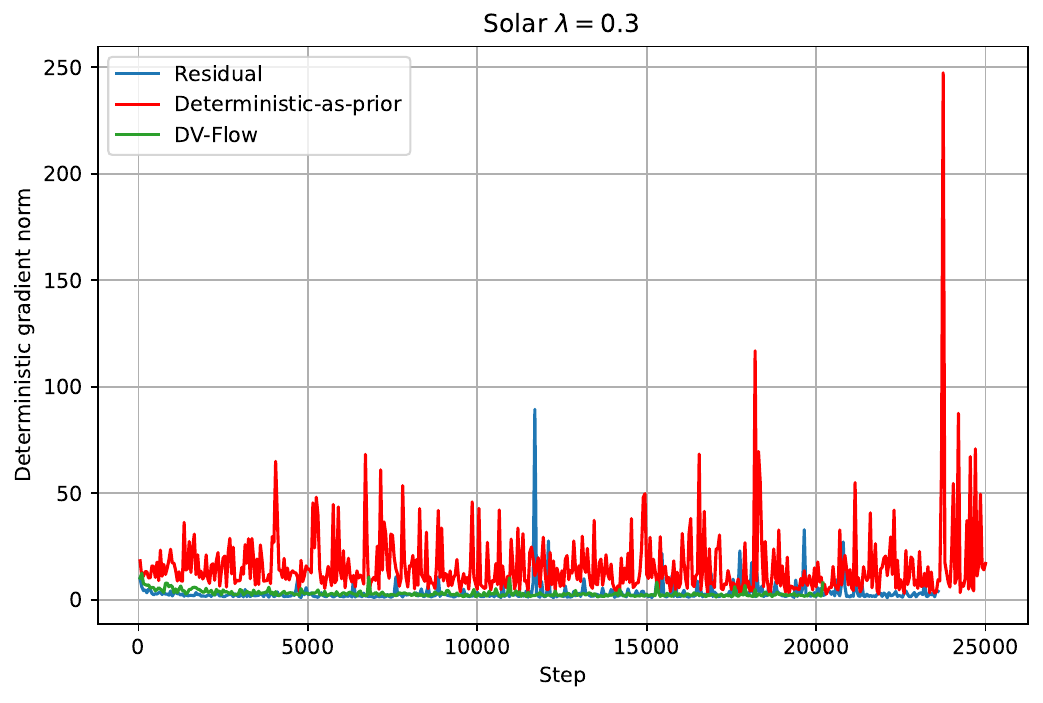}
\includegraphics[width=0.45\linewidth]{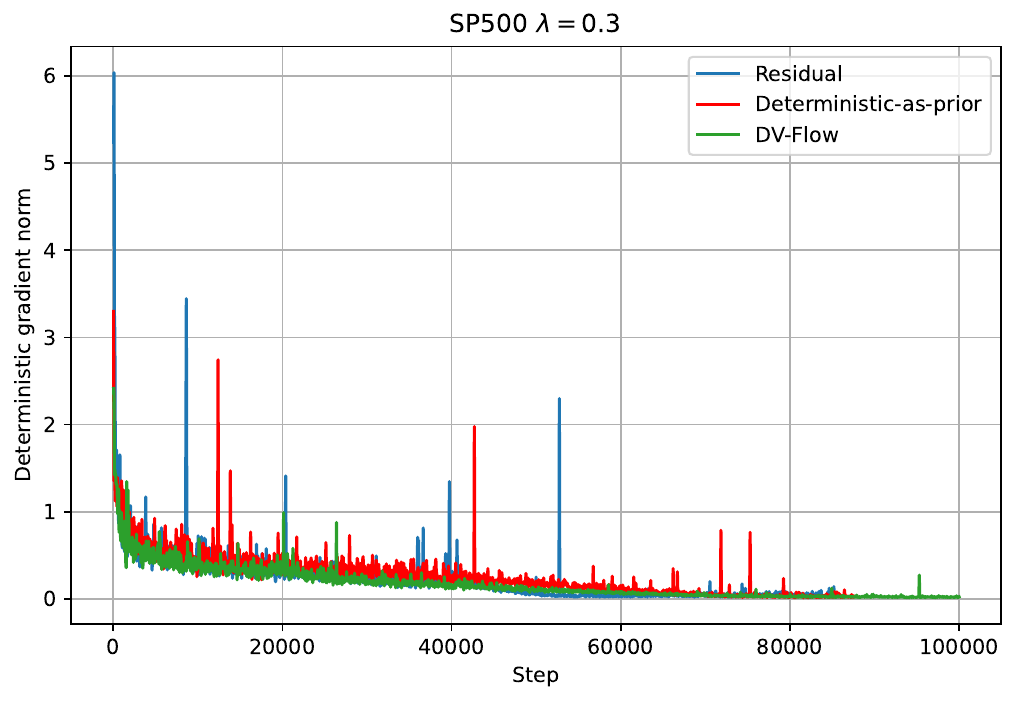}
\vspace{0.5em}
\includegraphics[width=0.45\linewidth]{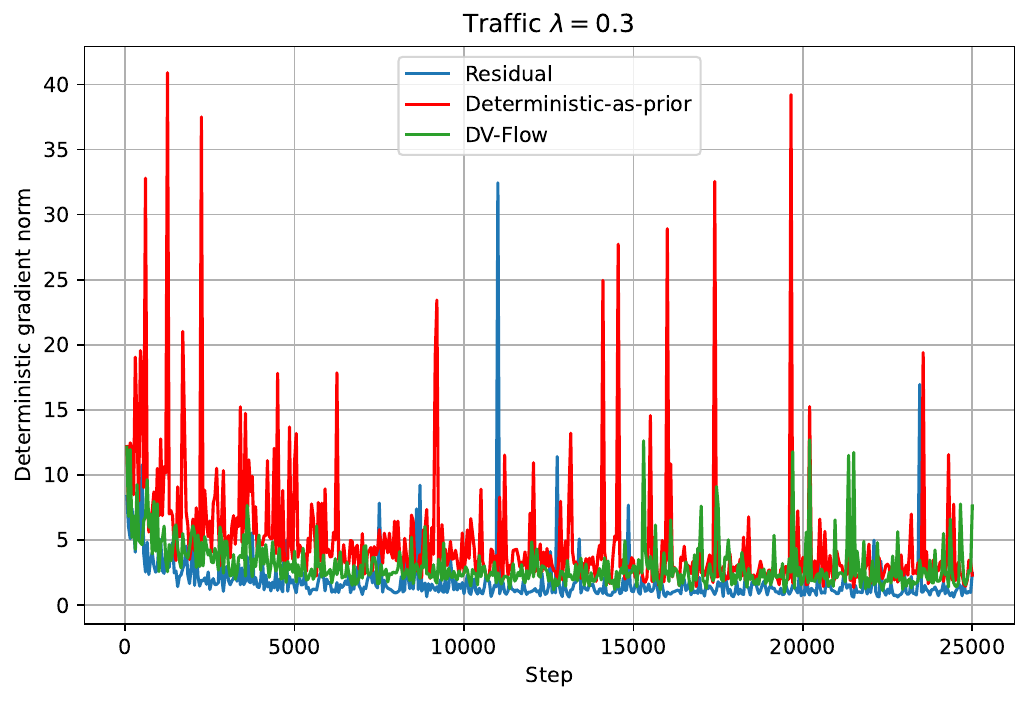}
\caption{\textbf{Gradient norm.}  We compare DV-Flow, the residual-based and deterministic-as-prior formulations with $\lambda=0.3$.  In all datasets, residual-based and deterministic-as-prior training exhibits large gradient spikes, whereas DV-Flow remains stable without large spikes.}
\label{fig:gradnorm_additional}
\end{figure}

\section{Visualization}
\begin{figure}[ht]
\centering
\includegraphics[width=0.5\columnwidth]{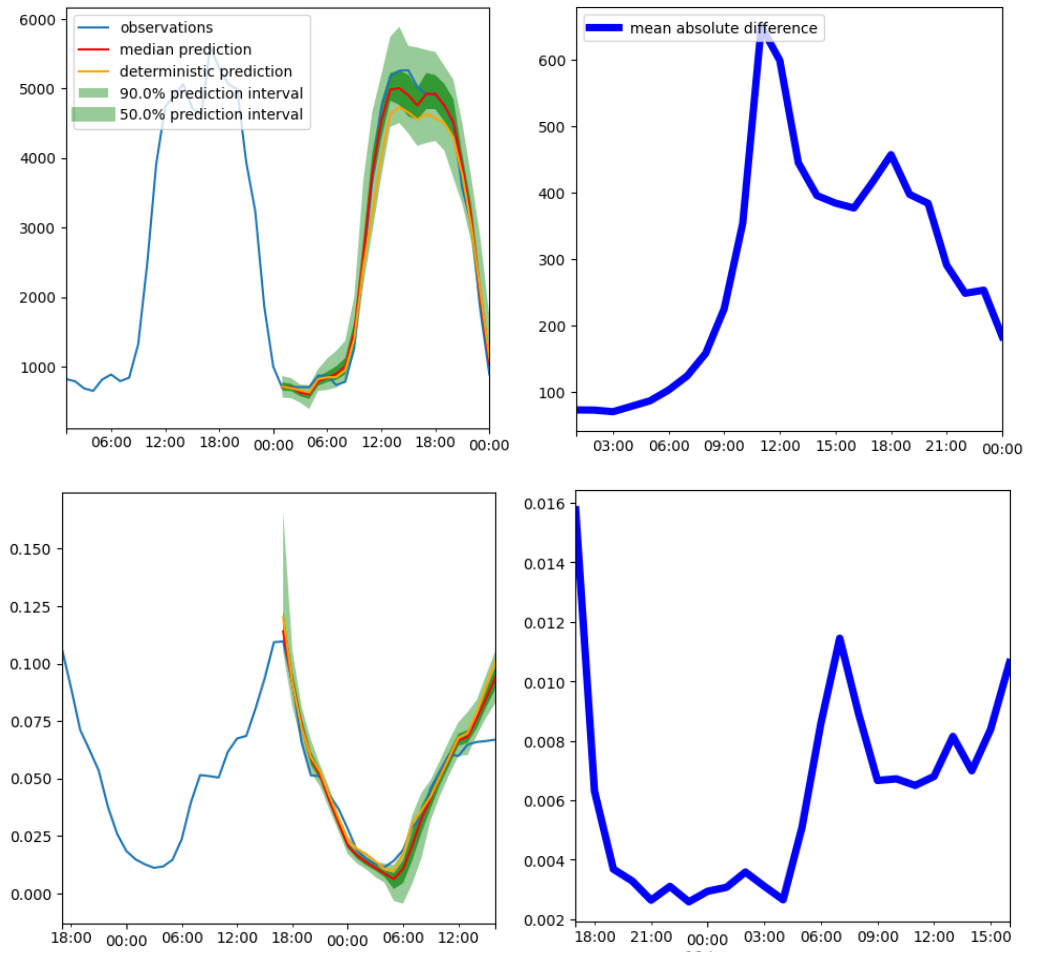}
\caption{Left: prediction interval; Right: MAE between the final prediction and deterministic prediction. Top: Electricity; Bottom: Traffic.}
\label{fig:visuals}
\end{figure}
In this section, we aim to gain a visual understanding of how DV-Flow generates its predictions and to examine the form of the resulting prediction bands. Figure~\ref{fig:visuals} plots predictive intervals and the mean absolute deviation between our final predictions and the deterministic prediction for two datasets. Variability does not simply grow toward the end of the forecast window; instead, it mirrors the variance present in the training data. For electricity, variance is highest around 12:00–18:00 when demand peaks and lowest overnight (21:00–06:00). Our model reproduces these patterns, yielding predictive distributions that closely match the empirical variability.

\section{Seasonal-Adjusted Time-Adaptive Normalization}
\label{sec:appendix_SATN}

we design Seasonal-Adjusted Time-Adaptive Normalization (SATN). We remove seasonal components that do not require a learnable module directly from the data, allowing the forecasting model to focus on learning the remaining temporal dynamics. Let $\bs$ denote seasonality:
\[
\by_t \leftarrow \by_t - \bs_t
\]
In typical continuous time series, residuals near the present—such as $\by_{t_0 + 1}$ —tend to have a lower variance and are closer in value to the most recent observation $\bx_{t_0}$ (Figure~\ref{fig:normalization} in Appendix, step 1, 3).

We design our normalization method from the perspective of Brownian motion. Specifically, we assume that the residual sequence $\by$ evolves similarly to a Brownian motion originating from the last step in the context window $\bx_{t_0}$. Under this assumption, we normalize $\by$ to resemble a Gaussian Process with approximately unit marginal variance at each time step. This facilitates stable training and better alignment with the prior distribution used in flow matching.

We normalize $\by$ as follows:
\begin{equation}
    \by_t \leftarrow \frac{\by_t - \bx_{t_0}}{\boldsymbol{\sigma}_t} 
\end{equation}
$\boldsymbol{\sigma}_t$ is computed on the basis of the linear interpolation property of Brownian motion, where the variance grows proportionally with time: $ \boldsymbol{\sigma}_{t_{0} + 1}^2 = \bbE {\big[(\by_{(t_0 + 1)} - \bx_{t_0}) ^2\big]}, \quad
 \boldsymbol{\sigma}_{T}^2 = \bbE{\big[(\by_{T} - \bx_{t_0})^2\big]} $.
Usually $\boldsymbol{\sigma}_{t_{0}+1}^2 < \boldsymbol{\sigma}_{T}^2$. To ensure that the variance increases monotonically with $t$, we estimate it using the following linear schedule $\boldsymbol{\sigma}_{t}^2 = (t - (t_0 + 1))\Delta + \boldsymbol{\sigma}^2_{(t_0 + 1)} , 
    \Delta = \frac{\alpha}{|T-(t_{0} +1)|} (\boldsymbol{\sigma}_{T}^2 - \boldsymbol{\sigma}_{(t_{0}+1)}^2)$ where 
$\alpha$ is a hyperparameter. All calculations are performed independently for each variable of the time series. 

For the context input $\bx_{1:t_0}$, we normalize it similarly to the prediction window $\by$ to maintain consistency in scale throughout the sequence. Specifically, for each $t < t_0$, we compute $\bx_t \leftarrow \frac{\bx_t - \bx_{t_0}}{\sigma_t}, ~~~~ \sigma_t^2 = \frac{t_0 - t}{t_0 - 1} \bbE[(\bx_{1} - \bx_{t_0})^2].$
We encode the context sequence $\bx_t$ and the prediction sequence $\by^{\tau}_t$ separately and then concatenate the encoded $E(\bx_t)$ from the context window and $E(\by^{\tau}_t)$ from the prediction window into a single sequence, before running it through the transformer.  The training algorithm can be found in Algorithm~\ref{alg:training}.

\begin{figure*}[t]
\vskip 0.2in
\begin{center}
\centerline{\includegraphics[width=\textwidth]{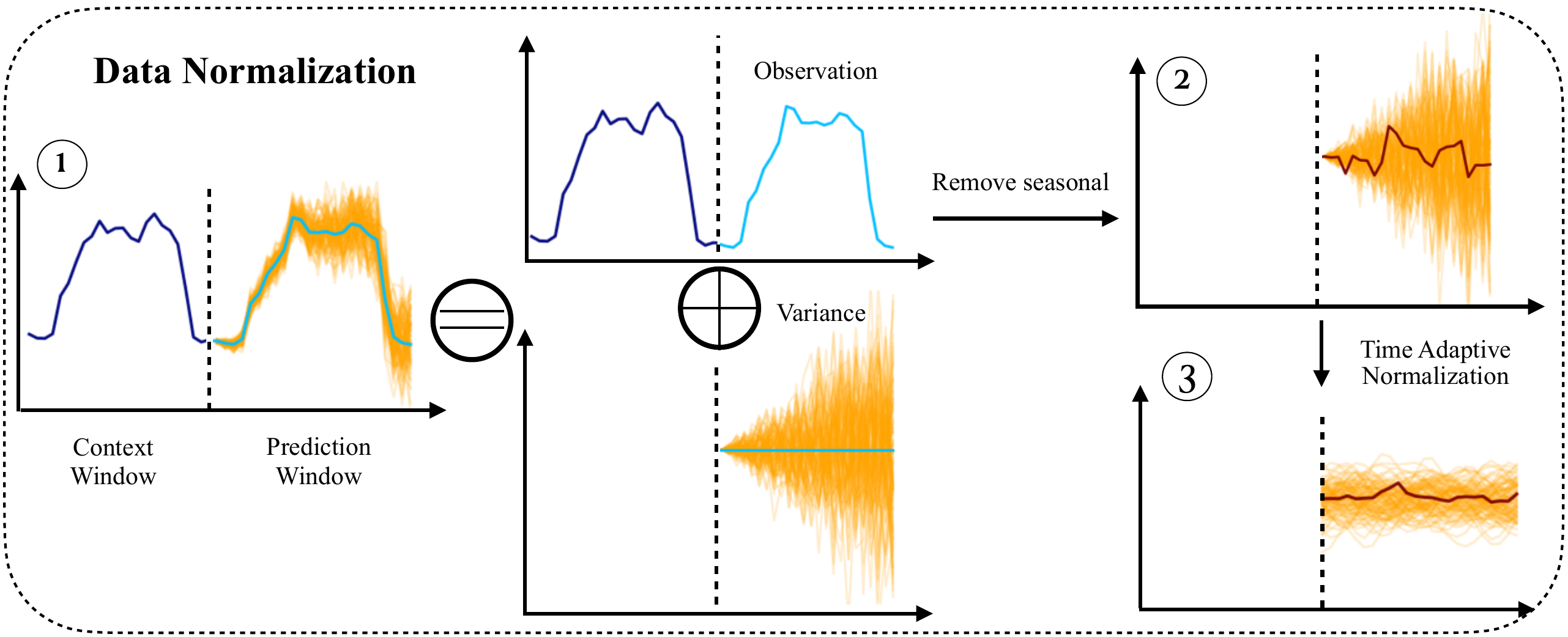}}
\caption{Data Normalization. Step 1: Assume the data follows Brownian motion after the last time step of the context window. Variance is smaller when close to the context window, and becomes bigger as time step increases. The data can be decomposed into two parts: the observation, and the variance. Step 2: Remove seasonal component. Step 3: Run time-adaptive normalization on the residual. At this point, the variance across all time steps are scaled to the same level.}
\label{fig:normalization}
\end{center}
\vskip -0.2in
\end{figure*}

The seasonal component of the time series often captures dominant low-frequency structures. Let $W$ denote the seasonal cycle length, which can be inferred from the data's periodicity. For example, in hourly time series with daily seasonality, $W=24$. Define $\alpha_t \in (1, \ldots, w)$ as the position of time $t$ within its seasonal cycle. The seasonal mean at cycle position $w$ is computed as: $\bu_{w} = \mathrm{mean}(\{\bx_{t}: \alpha_t = w\})$. We then define the deterministic seasonal estimate as: $\bs_{t} = \bu_{\alpha_t}$. This simple, nonparametric method provides a robust seasonal baseline, allowing the generative model to focus on learning the residual distribution.

We represent the seasonal cycle as $\textbf{s}* \in \mathbb{R}^{D \times w}$, where $\textbf{s}*$ is assumed to be drawn from a Gaussian distribution:
\begin{equation}
\textbf{s}* = \mathcal{N} (\bmu_s, \sigma_{\text{min}}I)
\end{equation}
When $\sigma_{\text{min}}$ approaches $0$, the seasonal cycle $\textbf{s}*$ will be very close to its mean $\bmu_s$. 
We estimate $\bs_{(t_0+1):T}$ as the empirical mean of the seasonal cycle. Specifically, we select an arbitrary time step within the cycle window $[1,w]$ and sample all corresponding data points $\mathcal{X} \in \mathbb{R}^{D \times w}$ from the training data that align with the chosen start position of the cycle window. For example, in hourly data, if hour $1$ is chosen as the start of the cycle window, we extract all $24$-length data points from different days that begin at hour 11. The mean $\bmu_s$ is then computed as $\bmu_s = \frac{1}{N}\sum \mathcal{X}$, where $N$ is the total number of sampled data points.

After obtaining the seasonal cycle $\textbf{s}*$, we perform adjustment and concatenation to reshape it to align with the current time series $\textbf{x}$. Suppose $\textbf{x}$ starts at position $m$ and ends at position $n$, where $m$ and $n$ are relative positions within the cycle window $[1:w]$. The adjustment involves shifting $\textbf{s}*$ such that the adjusted seasonal cycle $\textbf{s}^a$ is defined as $\textbf{s}^a = [\textbf{s}*[m:];\textbf{s}*[:m]]$, where $[;]$ denotes concatenation.

Let the length of $\textbf{x}$ be $l$. The time series $\textbf{x}$ spans $\lfloor l/w \rfloor$ full cycles, with an offset of $l \bmod w$ steps. The final full seasonal effect is constructed by fist concatenating $\textbf{s}^a$ repeated $\lfloor l/w \rfloor$ times to account for the full cycles, and then appending the first $l \bmod w$ steps of $\textbf{s}^a$ to account for the remaining partial cycle. A visual illustration of this process can be found in Figure~\ref{seasonal}.

\[
\textbf{s} = [\lfloor l/w \rfloor*[\textbf{s}^a];\textbf{s}^a[:l \bmod w]]
\]

\begin{figure*}[ht]
\vskip 0.2in
\begin{center}
\centerline{\includegraphics[width=\textwidth]{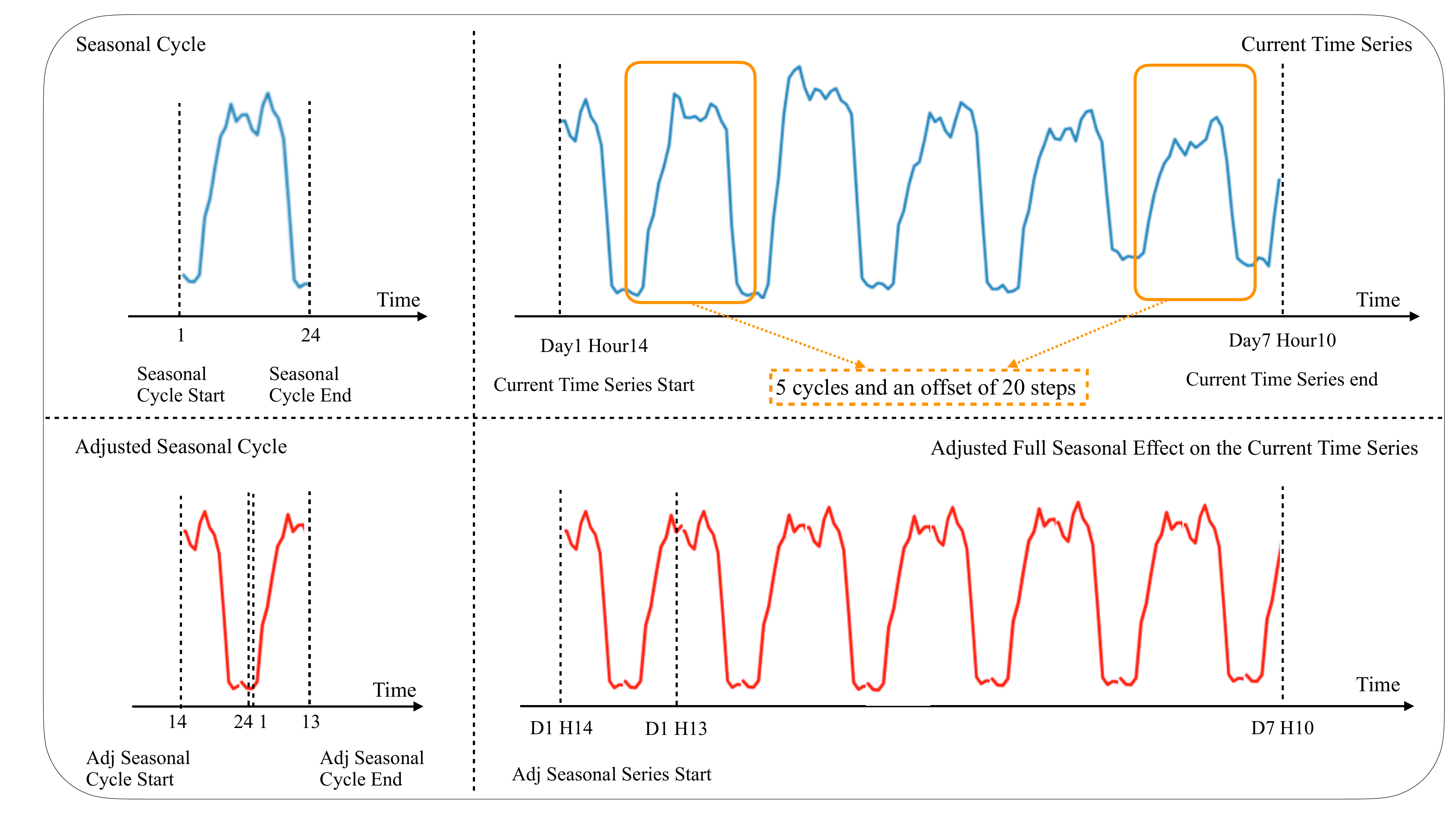}}
\caption{Top Left: Seasonal Cycle that starts at position 1 of the cycle window. Top Right: Current time series data that starts at position 14 of the cycle window, with 5 cycles and 10 offset steps. Bottom Left: Adjusted Seasonal Cycle that starts at position 14, same as the current time series. Bottom Right: Adjusted full seasonal effect on the current time series, which is the concatenation of 5 adjusted seasonal cycles and the first 10 steps from the adjusted seasonal cycle.}
\label{seasonal}
\end{center}
\vskip -0.2in
\end{figure*}

\section{More results on the toy dataset}
\label{app:more-results-on-toy-dataset}

In this toy dataset, $\bx$ is a random vector uniformly sampled from $[-1, 1]^3$. We compute the centers of two Gaussian mixtures with 
$\mu_1 = \bW_{1} \bx$, and $\mu_2 = d \frac{\bW_{2} \bx}{\|\bW_{2} \bx\|} + \mu_1$. Here $\bW_1$ and $\bW_2$ are two parameters randomly initialized from a standard Gaussian distribution. Note that the distance between the two centers is always $d$. Both mixtures have covariance $I$. The first mixture
has weight $0.7$, while the second misture has weight $0.3$. Therefore, $\bx$ decides both the mean and variance of the predictive distribution. Compared with 
the first toy dataset, the covarinace of the entire Gaussian mixture does not directly correspond to scale of $\bx$, therefore, it is harder to learn for the predictive model.

For both toy datasets, we train a DV-Flow model with a feedforward backbone, which has 3 hidden layers. The size of each hidden layer is fixed to 32. The model is trained with 1000 epochs as we can afford the computation. On each dataset, we test the model with 100 random runs. The performance is computed from the mean and variance from these runs.   

On this toy dataset, we also compare predictive samples of $\by$ for a single $\bx$ against the true predictive distribution. On this dataset, we see that $\lambda=0$ gives much better performances than other $\lambda$ choices. It is also much better than pre-training $f_\theta$ with point-wise loss. It worth more investigation how the performances of different $\lambda$ values are related to the distribution, but it is clear 
that the choice of $\lambda$ is important to the model performance.   

\begin{figure}[ht]
  \begin{tabular}{ll}
   \hspace{-8mm}
   \raisebox{5mm}{
   \scalebox{0.9}{
   \input{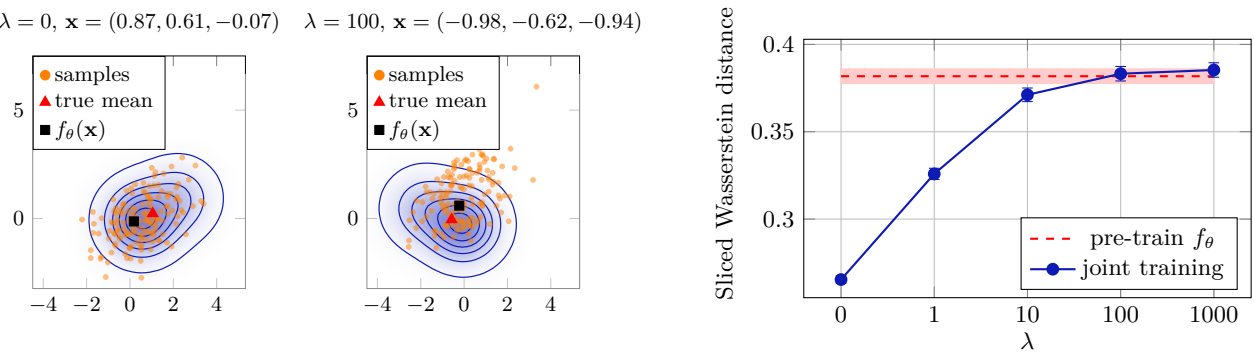} 
   }}
&
\definecolor{darkblue}{rgb}{0.05, 0.1, 0.7}
\begin{tikzpicture}
\begin{axis}[
  width=7.5cm, height=5cm,
  xlabel={$\lambda$},
  ylabel={Sliced Wasserstein distance},
  xlabel style={font=\small, yshift=1mm},
  ylabel style={font=\small, xshift=2mm},
  tick label style={font=\small},
  legend style={font=\small, at={(0.97,0.03)}, anchor=south east},
  xtick={0,1,2,3,4},
  xticklabels={0,1,10,100,1000},
  ymin=0.255,
  grid=both,
  grid style={line width=0.2pt, draw=gray!30},
  major grid style={line width=0.4pt, draw=gray!50},
  clip=false,
]

\addplot [red!20, fill=red!20, forget plot]
  coordinates {(0,0.377581) (4,0.377581)
               (4,0.385977) (0,0.385977)}
  -- cycle;

\addplot [red, dashed, thick, forget plot]
  coordinates {(0,0.381779) (4,0.381779)};
\addlegendimage{red, dashed, thick}
\addlegendentry{pre-train $f_\theta$}

\addplot [darkblue, thick, mark=*, mark size=2pt, error bars/.cd,
          y dir=both, y explicit]
  coordinates {(0,0.265372) +- (0,0.002150) (1,0.325865) +- (0,0.003219) (2,0.371139) +- (0,0.003894) (3,0.383210) +- (0,0.004144) (4,0.385341) +- (0,0.004136)};
\addlegendentry{joint training}

\end{axis}
\end{tikzpicture}
  \end{tabular}
  \caption{Investigation of the effect of point-wise loss on another
  toy dataset generated from Gaussian mixture models.}
  \label{fig:toy-example2}
\end{figure}


\end{document}